\documentclass{article}

\usepackage{microtype}
\usepackage{graphicx}
\usepackage{subcaption}
\usepackage{booktabs} 
\usepackage{tikz}
\usepackage{amsmath}
\usepackage{amsmath}
\usepackage{amssymb}
\usepackage{mathtools}
\usepackage{amsthm}
\usepackage{algorithm}
\usepackage{algorithmic}
\usepackage{thm-restate}
\usepackage{bbm}
\usepackage{xspace}
\usepackage{enumitem}
\usepackage{multirow}

\usepackage{hyperref}

\allowdisplaybreaks[4]
\usepackage[preprint]{icml2026}

\theoremstyle{plain}

\theoremstyle{definition}

\theoremstyle{remark}

\usepackage[textsize=tiny]{todonotes}

\usepackage{amsmath,amsfonts,amsthm,bm}

\def\eqref#1{equation~\ref{#1}}
\def\Eqref#1{Equation~\ref{#1}}

\def\1{\bm{1}}

\DeclareMathAlphabet{\mathsfit}{\encodingdefault}{\sfdefault}{m}{sl}
\SetMathAlphabet{\mathsfit}{bold}{\encodingdefault}{\sfdefault}{bx}{n}

\newcommand{\E}{\mathbb{E}}
\newcommand{\R}{\mathbb{R}}

\DeclareMathOperator*{\argmax}{arg\,max}

\newtheorem{defin}{Definition}

\newtheorem{thm}{Theorem}
\newtheorem{lem}{Lemma}

\newtheorem{prop}[thm]{Proposition}

\newtheorem{ass}[thm]{Assumption}

\newcommand{\N}{\mathbb{N}}

\newcommand{\Prob}{\mathbb{P}}

\newcommand{\bigot}{\widetilde{\mathcal O}}

\newcommand{\As}{\mathcal A}

\newcommand{\Fs}{\mathcal F}

\newcommand{\Os}{\mathcal O}

\newcommand{\Ss}{\mathcal S}

\newcommand{\Zs}{\mathcal Z}

\DeclareRobustCommand{\exavi}{\textsc{ExAVI}\@\xspace}
\DeclareRobustCommand{\exaq}{\textsc{ExAQ}\@\xspace}

\icmltitlerunning{Learning in Markov Decision Processes with Exogenous Dynamics}

\begin{document}

\twocolumn[
  \icmltitle{Learning in Markov Decision Processes with Exogenous Dynamics}



  \icmlsetsymbol{equal}{*}

  \begin{icmlauthorlist}
    \icmlauthor{Davide Maran}{equal,uni}
    \icmlauthor{Davide Salaorni}{equal,uni}
    \icmlauthor{Marcello Restelli}{uni}
  \end{icmlauthorlist}

  \icmlaffiliation{uni}{Department of DEIB, Politecnico di Milano, Milan, Italy}

  \icmlcorrespondingauthor{Davide Maran}{davide.maran@polimi.it}
  \icmlcorrespondingauthor{Davide Salaorni}{davide.salaorni@polimi.it}

  \icmlkeywords{Machine Learning, ICML}

  \vskip 0.3in
]



\printAffiliationsAndNotice{}  

\begin{abstract}
Reinforcement learning algorithms are typically designed for generic Markov Decision Processes (MDPs), where any state-action pair can lead to an arbitrary transition distribution. In many practical systems, however, only a subset of the state variables is directly influenced by the agent’s actions, while the remaining components evolve according to exogenous dynamics and account for most of the stochasticity.
In this work, we study a structured class of MDPs characterized by exogenous state components whose transitions are independent of the agent’s actions. We show that exploiting this structure yields significantly improved learning guarantees, with only the size of the exogenous state space appearing in the leading terms of the regret bounds. We further establish a matching lower bound, showing that this dependence is information-theoretically optimal. Finally, we empirically validate our approach across classical toy settings and real-world-inspired environments, demonstrating substantial gains in sample efficiency compared to standard reinforcement learning methods.
\end{abstract}

\section{Introduction}

While reinforcement learning (RL)~\citep{sutton1998reinforcement} has demonstrated remarkable success across a wide range of applications, its deployment in real-world settings remains a significant challenge~\citep{dulacarnold2019challenges}. Typically, RL problems are framed as Markov Decision Processes (MDPs)~\citep{puterman2014markov}, where an action is modeled as inducing a transition on the entire state, often disregarding the inherent independence of certain state variables.

In many practical scenarios, however, agents observe dynamics that are not under their direct control and evolve independently of the chosen action. This distinction is evident across numerous real-world domains: in finance, stock prices fluctuate independently of a retail investor's decisions; in reservoir management, controllers cannot influence weather conditions; and in energy systems, renewable generation is driven by environmental factors rather than operator intent. 
Crucially, these \emph{exogenous} phenomena hinder effective exploration and exacerbate the \emph{temporal credit assignment problem}~\citep{minsky1961steps}. Since the reward signal is often influenced by stochastic fluctuations from uncontrollable factors, it becomes difficult for the agent to disentangle the marginal contribution of its own actions from environmental noise. This low signal-to-noise ratio leads to high-variance gradient estimates, requiring a significantly larger number of samples to distinguish optimal policies from suboptimal ones. Furthermore, adopting standard exploration mechanisms to identify spurious correlations between actions and exogenous signals is pointless, since these dependencies are, by definition, nonexistent.
Furthermore, while these variables are essential components of the observation space, including them significantly inflates the state space complexity. In this scenario, the standard MDP formulation fails to explicitly separate uncontrollable variables from the controllable state space, thereby contributing to increased problem complexity, suboptimal exploration, and sample inefficiency.  

In this work, we investigate whether explicitly modeling this distinction can be leveraged to enhance learning. Unlike exogenous signals, the remaining \emph{controllable} portion of the environment (comprising \emph{endogenous} variables) typically evolves deterministically---or even stochastically, but in a reasonably known manner---in response to control actions. For instance, in algorithmic trading, while the asset price is stochastic and exogenous, the remaining budget is a deterministic function of the previous portfolio and the executed trade. By isolating the stochastic exogenous dynamics from the (often deterministic) controllable dynamics, we can design algorithms that better exploit the exogenous dynamics of the system.

To this end, we propose a novel refinement of the classical MDP framework that explicitly incorporates \emph{partial controllability}. Our formulation relies on a factorization of the state space into \textit{controllable} and \textit{uncontrollable} components, where the agent's actions cannot affect the evolution of the latter. We describe a mathematical framework building on MDPs and explore it from theoretical, algorithmic, and experimental perspectives. Specifically, we analyze the performance guarantees of both \textit{model-based} and \textit{model-free} approaches designed for this formulation, comparing them against their respective MDP-based counterparts in the \textit{finite-horizon tabular} setting. Finally, we empirically validate these algorithms on various toy environments and a real-world application involving complex endogenous and exogenous signals.

The main contributions of this work are summarized as follows:
\setlist{nosep}
\begin{enumerate}[noitemsep, topsep=0pt]
\item We introduce the \textit{Partially Controllable Markov Decision Process} (PCMDP), a structured extension of the classical MDP framework that explicitly distinguishes controllable from uncontrollable variables within the state space. 
\item We propose two algorithms tailored to this framework in the finite-horizon setting: (i) \emph{Exogenous-Aware Value Iteration} (\exavi), an enhanced version of Value Iteration (\textsc{VI})~\cite{sutton1998reinforcement} as a \emph{model-based} approach; and (ii) \emph{Exogenous-Aware Q-Learning} (\exaq), an extension of the classical Q-Learning (\textsc{QL}) algorithm~\citep{watkins1992q} as a \textit{model-free} solution. For both methods, we provide theoretical guarantees and regret bounds. 
\item We empirically validate these algorithms through comparative experiments on classical toy environments and a real-world domain, demonstrating their superior sample efficiency compared to standard MDP baselines. 
\end{enumerate}

The paper is organized as follows. In Section~\ref{sec:pcmdp}, we introduce the PCMDP framework. In Section~\ref{sec:algos}, we present the tailored algorithms with pseudocode and theoretical guarantees. In Section~\ref{sec:experiments}, we validate the aforementioned methods. Finally, in Section~\ref{sec:relatedworks}, we review existing literature connected to our work.

\section{Framework Formulation}\label{sec:pcmdp}
We start by recalling a few definitions about standard MDPs. A finite-horizon MDP is a tuple $M=(\Ss, \As, p, r, H)$, where $\Ss$ is the set of states, $\As$ is the set of actions, with cardinality $S$ and $A$, respectively, $p=\{p_h\}_{h=1}^{H-1}$ is the sequence of transition functions mapping, for each step $h \in [H-1] := \{1,\dots,H-1\}$, a pair $(s,a) \in \Ss \times \As$ to a probability distribution $p_h(\cdot |s,a)$ over $\Ss$, while the initial state $s_1$ may be arbitrarily chosen by the environment at each episode; $r=\{r_h\}_{h=1}^H$ is the sequence of reward functions mapping, for each step $h \in [H]$, a pair $(s,a)$ to a real number $[0,1]$, and $H$ is the horizon. 
In most applications, as in many theoretical papers, the reward function is known. Conversely, the transition function, which accounts for most of the complexity of the problem, is unknown.
At each episode $k\in [K]$, the agent chooses a policy $\pi_k = \{\pi_{h}^k\}_{h=1}^H$, which is a sequence of step-dependent mappings from $\Ss$ to probability distributions over $\As$. For each step $h \in [H]$, the action is chosen as $a_h\sim \pi_{h}^k(\cdot |s_h)$ and the agent gains reward $r_h(s_h,a_h)$; hence, the environment transitions to the next state $s_{h+1}\sim p_h(\cdot |s_h,a_h)$.

\textbf{Value Functions and Bellman Operators.}~~
The state-action value function (or \emph{Q-function}) quantifies the expected sum of the rewards obtained under a policy $\pi$, starting from a state-step pair $(s,h)\in\Ss\times [H]$ and fixing the first action at step $h$ to some $a\in\As$. Formally:
\begin{equation}\label{eq:action_state_val}
    Q_h^{\pi}(s,a) := \mathbb{E}_{\pi} \left[ \sum_{\ell=h}^{H} r_\ell(s_\ell,a_\ell)\bigg | s_h=s,a_h=a  \right],
\end{equation}
where $\E_{\pi}$ denotes expectation with respect to the stochastic process $a_h \sim \pi_h(\cdot|s_h)$ and $s_{h+1} \sim p_h(\cdot|s_h,a_h)$ for all $h \in [H]$.
The state value function (or \emph{V-function}) is defined as $V_h^\pi(s):=\E_{a\sim\pi_h(\cdot\vert s)}[Q_h^\pi(s,a)]$, for all $s\in\Ss$. The supremum of the value functions across all the policies is referred to as the optimal value function: $Q_h^\star(s,a):=\sup_\pi Q_h^\pi(s,a)$ for the Q-function and $V_h^\star(s):=\sup_{\pi}V_h^\pi(s)$ for the V-function.

An explicit way to find the optimal value function is given by the \textit{Bellman optimality operator}, which is defined, for every function $f:\Zs\to \R$, as
$\mathcal T_h f(s,a):=r_h(s,a)+\E_{s'\sim p_h(\cdot|s,a)}\left[\sup_{a'\in \As}f(s',a')\right].$
In fact, it is easy to show that $Q_h^\star = \mathcal T_h Q_{h+1}^\star$ at every step, while the optimal state-value function is obtained simply as $V_h^\star(s) = \sup_{a\in \As}Q_h^\star(s,a)$.\footnote{The existence of optimal policies is more subtle than in the finite-action case~\citep{bertsekas1996stochastic}, but this does not prevent us from defining a meaningful notion of regret.}

\textbf{Agent's Regret.}~~We evaluate the performance of an agent, i.e., of a policy $\pi_k$ played at episode $k\in[K]$, with its expected total reward, namely the V-function evaluated in the initial state $V^{\pi_k}_1(s_1^k)$. The goal of the agent is to play a sequence of policies $\{\pi_k\}_{k=1}^K$ to minimize the cumulative difference between the optimal performance $V^\star_1(s_1^k)$ and its performance $V_1^{\pi_k}(s_1^k)$, given the initial state $s_1^k$ chosen by the environment. This quantity takes the name of \textit{(cumulative) regret}, indicated as $R_K:=\sum_{k=1}^K \left(V_1^\star(s_1^k) - V_1^{\pi_k}(s_1^k) \right).$
The regret is a non-negative quantity and, by the normalization condition, it cannot exceed $HK$ as every term in the sum is bounded by $H$.
Notice that if $R_K=o(K)$, then the average performance of the chosen policies will converge to optimal performance for the number of training episodes growing to $+\infty$.
An algorithm choosing a sequence of policies with this property is called \emph{no-regret}.

\subsection{The PCMDP Framework} 
Motivated by real-life applications of RL, we study a subclass of MDPs that is built on an additional assumption to account for exogenous variables in the state space. As follows, we provide the formal definition of PCMDP.

\begin{defin}\textit{(PCMDP)}\label{def:pcmdp}
    A \emph{Partially Controllable MDP} is a Markov Decision Process whose state space $\Ss=\Ss^\bullet\times \Ss^\diamond$ and, at any time step $h$, the transition function writes as:
    \begin{align*}
        p_h(s_{h+1}|s_h,a_h) &= p_h^\diamond(s_{h+1}^\diamond|s_h^\diamond,s_h^\bullet,a_h) p_h^\bullet(s_{h+1}^\bullet|s_h^\bullet)
        \\ \text {with }s_h&=[s_h^\diamond,s_h^\bullet]\in \Ss^\diamond\times \Ss^\bullet=\Ss.
    \end{align*}
\end{defin}
We call $\Ss^\diamond$ the endogenous \emph{controllable} part, $\Ss^\bullet$ the exogenous \textit{uncontrollable} part, and $S^\diamond,S^\bullet$ their cardinalities, respectively.

As the naming convention suggests, $s^\bullet$ represents the part of the state outside the agent's control, since its transition does not depend on the agent's actions. The dynamics of $s^\bullet$ are unknown and may be very complex. On the other side, $s^\diamond$ contains the variables that are influenced by everything else, including the actions. We assume it has a simple (even if possibly stochastic) dynamics that the learner knows. 
\begin{ass}\label{ass:pdiam}
    The agent has full knowledge of $p_h^\diamond$, for every $h\in [H]$.
\end{ass}
While this choice may seem restrictive, it reflects the dynamics of many real-world problems. Hereafter, we provide a few examples from real-world domains. Further examples are provided in Appendix~\ref{app:realworld}.

\paragraph{Trading.} Consider, for instance, a \emph{trading problem}. While we explicitly refer to the optimal execution setting~\citep{almgren2001optimal}, the interaction of a retail investor with the market can be naturally framed as a PCMDP. Suppose we trade a single stock at discrete time steps. 
At each time step $t$, the agent can decide whether to \textit{buy} or \textit{sell} a specific number of stock units. The state of the problem is given by $s_t = (\bm {\omega}_t, u_t, b_t)$, where $\bm \omega$ is some information vector that contains the current price $\omega_t$ plus some relevant features for its evolution (such as market volumes, volatility in the last five minutes, or past values for the price). $u_t$ is the number of stock units currently held by the agent, and $b_t$ is the agent's available budget for purchasing stocks. The agent selects an action $a_t \in \mathbb{Z}$, where $a_t > 0$ corresponds to buying  $a_t$ units, and $a_t < 0$ to selling $|a_t|$ units. The transitions for the agent's internal state variables are deterministic and governed by:
$$
u_{t+1}=u_t+a_t,\qquad c_{t+1}=c_t-a_tp_t.
$$
Conversely, the vector $\bm \omega_{t+1}$, that contains $\omega_{t+1}$ evolves exogenously,
depending solely on $\bm \omega_t$. 
At first glance, this decomposition assumes that the agent's individual trading behavior has no impact on market dynamics. Nonetheless, the work by \citep{almgren2001optimal} shows that, under mild assumptions, both permanent and temporary impact can be implicitly modeled by changing the reward function while keeping a price transition that is not influenced by the agent.
Accordingly, the state can be partitioned into a controllable component $s_t^\diamond = (u_t, c_t)$ and an uncontrollable component $s_t^\bullet = (\bm \omega_t)$. This decomposition aligns naturally with the PCMDP framework, as only a specific subset of the state space is influenced by the agent’s actions.

\paragraph{Reservoir Management.} A second example is \emph{water reservoir management}. In this control task, the goal is to regulate the outflow from a water reservoir (via a dam) to meet daily water demand while avoiding overflow when the reservoir's capacity is exceeded. The state of the system at time $t$ is represented as $s_t = (l_t, q_t)$, where $l_t$ is the current water level in the reservoir and $q_t$ is the net natural inflow of water, which includes rainfall, evaporation, and environmental effects. The agent selects an action  $a_t \geq 0$, representing the amount of water released. The system dynamics are then given by:
$$
l_{t+1}=l_t+q_t-a_t,
$$
while the inflow $q_{t+1}$ evolves according to an exogenous stochastic process dependent on environmental factors and weather conditions (not the agent's decisions).
Similarly to the trading case, this environment exhibits partial controllability: the agent can influence only the water level $l_t$, while the inflow $q_t$ evolves independently. Thus, the state can be decomposed into $s_t^\diamond = (l_t)$ and $s_t^\bullet = (q_t)$.\footnote{The net inflow could theoretically depend on the controllable state, since evaporation rates depend on surface area and water level $l_t$. However, for reservoirs with approximately vertical walls (constant surface area) and with negligible evaporation losses compared to tributary inflows, our exogenous assumption holds.}

\section{Algorithms}\label{sec:algos}
In this section, we present PCMDP-specific algorithms with theoretical regret guarantees, including a \emph{model-based} and a \emph{model-free} approach.

\subsection{Model-based Approach: \exavi}
Model-based RL algorithms estimate the environment's dynamics, i.e., the transition probabilities, and leverage this model to plan optimal actions, often via Value Iteration \cite{sutton1998reinforcement}. These methods are sample-efficient because they reuse experience to simulate possible futures. Among these methods, \textsc{UCBVI} (Upper Confidence Bound Value Iteration) by~\citet{azar2017minimax} adds optimistic bonuses to value estimates based on uncertainty, encouraging exploration in a principled way. It balances exploration and exploitation by favoring actions with either high estimated value or high uncertainty, and it achieves provably efficient regret bounds in tabular MDPs.

In this section, we present our \emph{Exogenous--Aware Value Iteration} (\exavi), a novel algorithm which improves the guarantees of \textsc{UCBVI} by exploiting the structure of the PCMDP setting. First, assuming to already know the endogenous dynamics of $p^\diamond$, \exavi needs to estimate only the part of the transition corresponding to $p^\bullet_h$, in contrast to \textsc{UCBVI}, which estimates $p_h$. Second, and more interesting, \exavi works without any bonus term: no \emph{optimism} is needed. Intuitively, this is because the epistemic uncertainty resides entirely within the exogenous dynamics $p^\bullet_h$. Since the evolution of $s^\bullet$ is independent of the agent's actions, "active" exploration is unnecessary; the agent simply observes the exogenous process unfold regardless of its policy.

\paragraph{Algorithm Structure.} Fix $h\le H, k\le K$. 
For any state $s^\bullet\in \Ss^\bullet$, let $n_h^k(s^\bullet)$ denote the visitation count up to episode $k$: 
$$
n_h^k(s^\bullet):=\sum_{\tau=1}^k \mathbbm{1}\{s_h^{\bullet,\tau}=s^\bullet\},
$$
\begin{equation} 
    \widehat{p}_h^{\bullet,k}(\bar{s}^{\bullet}|s^\bullet) := \frac{1}{n_h^k(s^\bullet)} \sum_{\tau=1}^{k} \mathbbm{1}\{s_h^{\bullet,\tau}=s^\bullet, s_{h+1}^{\bullet,\tau}=\bar{s}^{\bullet}\}. \label{eq:phat}
\end{equation}
This definition is central to our analysis: standard concentration arguments show that $\widehat p_h^{\bullet,k}(\bar s^{\bullet}|s^\bullet)$ converges to $p^\bullet_h$ at an optimal rate, the more samples we collect. In contrast, the standard \textsc{UCBVI} estimator of the whole $p_h$ in a general MDP is defined as follows:
\begin{equation} 
\widehat p_h^{k}(\bar s|s,a) := \frac{\sum_{\tau=1}^{k}\mathbbm{1}\{s_h^\tau=s, a_h^\tau=a,s_{h+1}^\tau=\bar s\}}{\max\{1, \sum_{\tau=1}^{k}\mathbbm{1}\{s_h^\tau=s, a_h^\tau=a\}\}}. 
\end{equation}
The denominator of the previous formula lies in sharp contrast with our formula in \Eqref{eq:phat}. Here, it depends on the agent's actions; therefore, exploration is needed for the estimator to converge. This is the reason why \textsc{UCBVI} needs optimism, while our \exavi does not.

\begin{algorithm}[tb]
\caption{Exogenous--Aware Value Iteration (\textsc{ExAVI})}
\label{alg:exavi}
\begin{algorithmic}[1]
\REQUIRE Controllable model $p^\diamond$, discrete spaces $\Ss^\diamond,\Ss^\bullet,\As$, reward function $r_h$ for all $h \in [H]$.
\STATE Initialize $\widehat{p}^{\bullet, 1}$ and $\pi^1$ arbitrarily.
\FOR {$k=1,\dots,K$}
    \FOR {$h=1,\dots H$}
        \STATE $a_h^k=\pi_h^k(s_h^k)$        
    \ENDFOR
    \STATE Update $\widehat p_h^{\bullet,k}$ with~\Eqref{eq:phat} using  $\{s_h^{\bullet,k}\}_{h=1}^H$
    \STATE $\widehat p_h^{k}(\bar s|s,a)\gets p_h^\diamond(\bar s^\diamond|s,a)\widehat p_h^{\bullet,k}(\bar s^\bullet|s^\bullet) $ 
    \STATE $\{\widehat Q_h^{k}\}_{h=1}^H\gets$ \textsc{ValueIteration}$(\widehat p_h^{k}, r_h)$
    \STATE $\pi_h^{k+1}(\cdot)=\argmax_{a\in \As}\widehat Q_h^{k}(\cdot,a)$
\ENDFOR
\end{algorithmic}
\end{algorithm}

\paragraph{Theoretical Guarantees.} As anticipated, \textsc{UCBVI} can achieve theoretical guarantees on the regret by facing, with optimism, the \emph{exploration-exploitation} dilemma. Its regret (using Hoeffding-type bonuses) scales as $\bigot(H^2\sqrt{SAK}+H^3S^2A)$, with high probability. In our case, where $S=S^\bullet S^\diamond$, this writes as $\bigot(H^2\sqrt{S^\bullet S^\diamond AK}+H^3(S^\bullet S^\diamond )^2A)$. In contrast, our algorithm enjoys the following regret bound.
\begin{restatable}{thm}{regretboundone}
    Under Definition \ref{def:pcmdp}, with probability at least $1-\delta$, the regret of \exavi~(Algorithm~\ref{alg:exavi}) satisfies:
    $$R_K\le \bigot\left(H^2\sqrt{S^\bullet K}+H^3{S^\bullet}^2\right).$$
\end{restatable}
As we can notice, many terms have disappeared. We reduced $H^2\sqrt{S^\bullet S^\diamond AK}$ to $H^2\sqrt{S^\bullet K}$ and $H^3(S^\bullet S^\diamond )^2A$ to $H^3{S^\bullet}^2$. The terms $S^\diamond$ and $A$ have not completely disappeared, but they only appear in logarithmic terms, which are not visible due to the $\bigot$ notation. This descends from a stronger assumption about the model and renders learning more efficient by considering only exogenous variables, as the algorithm's name suggests. The formal, detailed proof is provided in Appendix~\ref{app:proof-exavi}.

\subsection{Model-free Approach: \exaq}
A model-free algorithm learns to act directly from experience without explicitly estimating the environment's dynamics. Instead of building a model of transitions, it learns value functions or policies from sampled trajectories. Q-Learning (\textsc{QL})~\citep{watkins1992q} is one of the most relevant algorithms of this family. In the context of standard MDPs, its regret guarantee is similar to that of \textsc{UCBVI} that we mentioned in the previous section.

Adapting this algorithm to enjoy an improved regret bound in our setting is more difficult. While in model-based approaches we can fix $p_h^\diamond$, which is known, and use samples to learn only $p_h^\bullet$, here we do not want to estimate the transition functions at all. Therefore, leveraging the knowledge of $p_h^\diamond$ is more challenging. Still, we can bypass the transition function by working directly on the empirical Bellman operator. One crucial step in the proof of the regret bound of \textsc{QL} (provided by \citet{jin2018q}) is the definition of the empirical operator $P_h^{k}[f]:=f(s_{h+1}^k)$, mapping a function $f:\Ss\to \R$ to a real number. With $s,a=s_h^k,a_h^k$, this empirical operator is an unbiased estimate of the value of the next state, as $\E[f(s_{h+1}^k)|s_h^k,a_h^k]=\E_{\bar s\sim p_h(\cdot|s,a)}[f(\bar s)]$. In our setting, the former operator assumes a more complex form, which is now a function of $s,a$ (only the dependence in $s^\diamond,a$ appears, as $s^\bullet$ is fixed to $s_h^{\bullet,k}$ by definition).
\begin{equation}
    \textbf{P}_h^{k}[f](s,a):=\sum_{\bar s^\diamond\in \Ss^\diamond}f(\bar s^\diamond,s_{h+1}^{\bullet,k})p_h^\diamond(\bar s^\diamond|s,a)\label{eq:pk}.
\end{equation}
For $h=H$, when the next state is not defined, we put $\textbf{P}_h^{k}[f](s,a)=0$ by convention. To understand \Eqref{eq:pk}, let us take any pair $s,a$ such that $s^\bullet=s^{\bullet,k}$. The following equations hold:
\begin{align*}
    \E[f(\bar s)|s,a]&=\sum_{\bar s^\diamond\in \Ss^\diamond}\sum_{\bar s^\bullet\in \Ss^\bullet}f(\bar s^\diamond,\bar s^{\bullet})p_h^\diamond(\bar s^\diamond|s,a)p_h^\bullet(\bar s^\bullet|s^\bullet)\\
    =&\sum_{\bar s^\bullet\in \Ss^\bullet}\left[\sum_{\bar s^\diamond\in \Ss^\diamond}f(\bar s^\diamond,\bar s^{\bullet})p_h^\diamond(\bar s^\diamond|s,a)\right]p_h^\bullet(\bar s^\bullet|s^\bullet)\\
    =&\E\left[\textbf{P}_h^{k}[f](s,a)\right].
\end{align*}
Crucially, this time we do not need to fix both the state and the action, but just the uncontrollable part of the state, i.e., $s_h^{\bullet,k}=s^\bullet$. 

\paragraph{Algorithm Structure.} As happens in \exavi, also \exaq works without any explicit exploration bonus. In fact, the independence of the exogenous part allows us to always act with the estimated best policy. The central innovation of the algorithm, which enables significantly lower sample complexity than standard \textsc{QL}, is encapsulated in \Eqref{eq:qlj}. Unlike conventional \textsc{QL}, which executes a single update per temporal step, the proposed method updates the estimated $Q$-function across $S^\diamond A$ state-action pairs simultaneously.
While the exploration regards only the exogenous factors, we can reuse the same information to update every possible controllable configuration with the same $s^{\bullet,k}$. Consistently, the choice of the learning which achieves the regret bound is
\begin{equation}
    \alpha_t=\frac{H+1}{H+t}\qquad t=n_h^k(s_h^{\bullet,k})\label{eq:lr}.
\end{equation}
The former does not depend on the number of visits of the state-action pair, but only on that of $s^{\bullet}$.
It is worth noting that the algorithm involves a triple-nested loop over $\Ss, \As$ and $[H]$. This would scale poorly with the size of state and action spaces. However, this bottleneck can be mitigated by vectorizing the value updates. By processing the entire $Q$-table simultaneously, we can obtain direct access to all $(s^\diamond,a)$ pairs, effectively reducing the iterative structure to a single loop over the episode steps $[H]$.

\paragraph{Theoretical Guarantees.} As before, our approach allows us to erase the dependency on $\Ss^\diamond$ and $A$ in the regret bound (except for logarithmic terms). While the standard analysis of \textsc{QL} with Hoeffding's bonuses achieves a regret of order $\bigot(\sqrt{H^5SA K})=\bigot(\sqrt{H^5S^\bullet S^\diamond A K})$, our algorithm enjoys the following.
\begin{restatable}{thm}{regretboundtwo}
    Under Definition \ref{def:pcmdp}, with probability at least $1-\delta$, the regret of \exaq~(Algorithm~\ref{alg:exaq}) when run with the learning rate in \Eqref{eq:lr} is bounded by
    $$R_K\le \bigot\left(\sqrt{H^5S^\bullet K}\right).$$
\end{restatable}
As expected, the regret substantially improves if $\Ss^\diamond$ or $A$ are non-trivial. The extended proof of this theorem is provided in Appendix~\ref{app:proof-exaq}.

\paragraph{Lower Bound.} To conclude the theoretical analysis, we prove that the previous results constitute the best possible regret bound, in terms of dependence on $K$ and state space size.

\begin{restatable}{thm}{lowerbound}\label{thm:lowerbound}
    For any algorithm, there is a PCMDP instance (definition \ref{def:pcmdp}) where it suffers regret of
    $$\E[R_K]\ge \Omega(\sqrt{S^\bullet K}).$$
\end{restatable}
Theorem \ref{thm:lowerbound} shows that a square-root scaling in the size of the uncontrollable states cannot be avoided. Formal proof of the lower bound is reported in Appendix~\ref{app:lowerbound}.

\begin{algorithm}[t]
\caption{Exogenous--Aware Q-Learning (\textsc{ExAQ})}
\label{alg:exaq}
\begin{algorithmic}[1]
\REQUIRE  Controllable model $p^\diamond$, discrete spaces $\Ss^\diamond,\Ss^\bullet,\As$, learning rates routine $\{\alpha_t\}_{t=1}^K$.
\FOR{$s\in \Ss,a\in \As,h\in [H]$}
    \STATE Initialize $Q_h^0(s,a)\gets(H-h)$
\ENDFOR
\FOR{$k=1,\dots,K$}
    \STATE $Q_h^k(s,a) \gets Q_h^{k-1}(s,a), \quad  \forall s \in \Ss, a \in \As$
    \STATE $V_h^k(s,a) \gets \max_{a\in \As}Q_h^{k}(s,a), \quad  \forall s \in \Ss$
    \STATE $\pi_h^k(s)=\argmax_{a\in \As}Q_h^k(s,a),  \quad \forall s \in \Ss, h \in [H]$
    \FOR{$h=1,\dots H$}
        \STATE $a_h^k\gets\pi_h^k(s_h^k)$        
    \ENDFOR
    \FOR{$h=1,\dots,H$}
        \STATE  Let $t = n_h^k(s^{\bullet,k}_h)$
        \FOR{$s^\diamond \in \Ss^\diamond, a\in \As$}
            \STATE Construct full state $s = (s^\diamond, s^{\bullet,k}_h)$
            \STATE Compute target $\bm w = r_h(s,a) + \textbf{P}_h^{k}[V_{h+1}^k](s,a)$
            \STATE Update $Q_h^{k}(s,a) \gets (1-\alpha_t)Q_h^k(s,a) + \alpha_t \bm w$ \label{eq:qlj}
        \ENDFOR
    \ENDFOR
\ENDFOR
\end{algorithmic}
\end{algorithm}

\section{Experiments}\label{sec:experiments}
To rigorously prove the capabilities of our novel algorithms in exploiting the PCMDP structure, we corroborate the theoretical analysis with an experimental campaign. In total, we compare \exavi and \exaq against their MDP-based versions, i.e., \textsc{UCBVI} and \textsc{QL}, across three environments. Hereafter, we will present only two scenarios, while the remaining one is reported in Appendix~\ref{app:elevator-details}.

\paragraph{Taxi with Traffic Environment.} The first environment we consider is a variation of a widely recognized benchmark originally proposed by~\citet{dietterich1999hierarchical} and included in the Gymnasium collection of toy text environments~\cite{towers2025gymnasium}: the \texttt{Taxi}. We selected this environment due to its tractable state-action space, which facilitates the evaluation of model-based algorithms such as \textsc{UCBVI} and \textsc{ExAVI}, whose space complexity scales very fast. The task involves a taxi agent navigating a $5\times 5$ grid world to pick up and deliver passengers at four designated locations. The agent has 6 actions: move in four directions, pick up, and drop off. Rewards are: $+20$ for a correct delivery, $-10$ for an invalid pickup or drop-off, and $-1$ for all other steps to encourage faster learning.

To align this scenario with the PCMDP framework, we introduced specific modifications to the environment's dynamics. Whereas the standard episode terminates upon a single successful delivery, we enforce a fixed time horizon, redefining the objective as maximizing the total number of deliveries within the allotted time. To incorporate an exogenous signal, we simulated stochastic traffic congestion at specific "choke points" on the grid, which impedes the taxi's movement. At each time step, the agent observes the binary traffic status of these locations, requiring it to adapt its routing strategy based on the learned traffic distribution. The controllable state at step $h$ of each episode is defined as $s^\diamond_h = (x_h, y_h, \psi_h, d_h)$, where $x_h$ and $y_h$ represent the row and column position of the taxi within the grid, $\psi_h$ the position of the passenger, and $d_h$ the current destination. Instead, the exogenous part is given by the boolean traffic vector at the $N_\text{tr}=3$ congested locations, such that $s^\bullet_h = (\boldsymbol{b}^\text{tr}_h)$ with $\boldsymbol{b}^\text{tr}_h \in \{0,1\}^{N_\text{tr}}$, where the $i$-th component $b^\text{tr}_{h,i}=1$ indicates that the $i$-th choke point is blocked at step $h$
The dynamics of the traffic vector are governed by an independent Bernoulli process: at each step $h$, every component $b_{h,i}$ is sampled independently according to a fixed congestion probability $p_{\text{tr}}$, such that $b_{h,i} \sim \text{Bernoulli}(p_{\text{tr}})$.The action space and reward functions remain the same of the original environment.
Figure~\ref{fig:taxi-screenshot} shows a screenshot of the modified environment. Additional details on \texttt{TaxiEnv} are provided in Appendix~\ref{app:env-details}.

\paragraph{Optimal Execution Environment.} This environment serves to validate our approach in a significant real-world scenario: the optimal execution of portfolio transactions. We designed this environment based on the seminal work of~\citet{almgren2001optimal}, introducing minor modifications to the reward function and price profile generation. The problem is defined as follows: an agent is tasked with liquidating a specific inventory of securities by the end of a trading day, which is discretized into $N$ fixed time intervals. At each step $h$, the agent observes the current inventory level and the market asset price, then determines the target inventory level for the subsequent step.
This task naturally aligns with the PCMDP framework, as the inventory level constitutes the controllable state component, whereas the asset price evolves as an exogenous process. Formally, we define the controllable state as $s^\diamond_h = (u_h)$, where $u_h \in \mathcal{U}$ denotes the current inventory, and the uncontrollable state as $s^\bullet_h = (\omega_h) \in \mathcal{P}_\omega$, representing the asset price. The action is defined as the target inventory for the next step, $a_h = u_{h+1}$. Consistent with the retail investor assumption, we posit that the agent's trading decisions do not impact the market price.

In our experimental setup, following~\citet{almgren2001optimal}, we define the inventory space $\mathcal{U}$ as the discrete set of integers $\{0, \dots,U_\text{max}\}$ with $U_\text{max} = 100$, resulting in a cardinality of $|\mathcal{U}|=101$. The action space $\mathcal{A}$ is isomorphic to $\mathcal{U}$, meaning the agent can choose to hold any valid amount of securities for the next step (though typically $a_h \leq u_h$ in liquidation tasks). The market price space $\mathcal{P}_\omega$ is discretized into a grid of $1000$ distinct values with a fixed granularity (tick size) of $0.02$. The instantaneous reward $r_h$ is composed of three terms: execution cost, holding risk penalty, and transaction revenue. First, the execution cost $c_{\text{ex},h}$, derived by the original paper, accounts for temporary market impact and fixed transaction fees. It is defined as:
\begin{equation*}
    c_{\text{ex},h} = \epsilon |n_h| + \frac{\tilde{\eta}}{\tau}n_h^2,
\end{equation*}
where $n_h = u_h - u_{h+1}$ is the number of shares traded, $\epsilon \in \mathbb{R}^+$ represents fixed transaction costs, $\tilde{\eta} \in \mathbb{R}^+$ is the adjusted temporary impact parameter, and $\tau \in \mathbb{R}^+$ is the time interval length.
Second, the holding risk penalty $c_{\text{hold},h}$ quantifies the risk of maintaining an inventory position in the face of market volatility. Following standard mean-variance optimization, this is defined as:
\begin{equation*}
    c_{\text{hold},h} = \lambda \tau \sigma^2 u_{h+1}^2,
\end{equation*}
where $\lambda \in \mathbb{R}^+$ is the risk aversion parameter and $\sigma \in \mathbb{R}_{\geq 0}$ is the asset volatility.
Finally, the transaction revenue (or cost) $\varrho_h$ is given by $\varrho_h = n_h \cdot p_h$. The total reward is then given by $r_h = \varrho_h - c_{\text{ex},h} - c_{\text{hold},h}$.

\begin{figure}[th]
    \centering
    \begin{subfigure}[b]{\linewidth}
        \centering
        \includegraphics[width=0.86\linewidth]{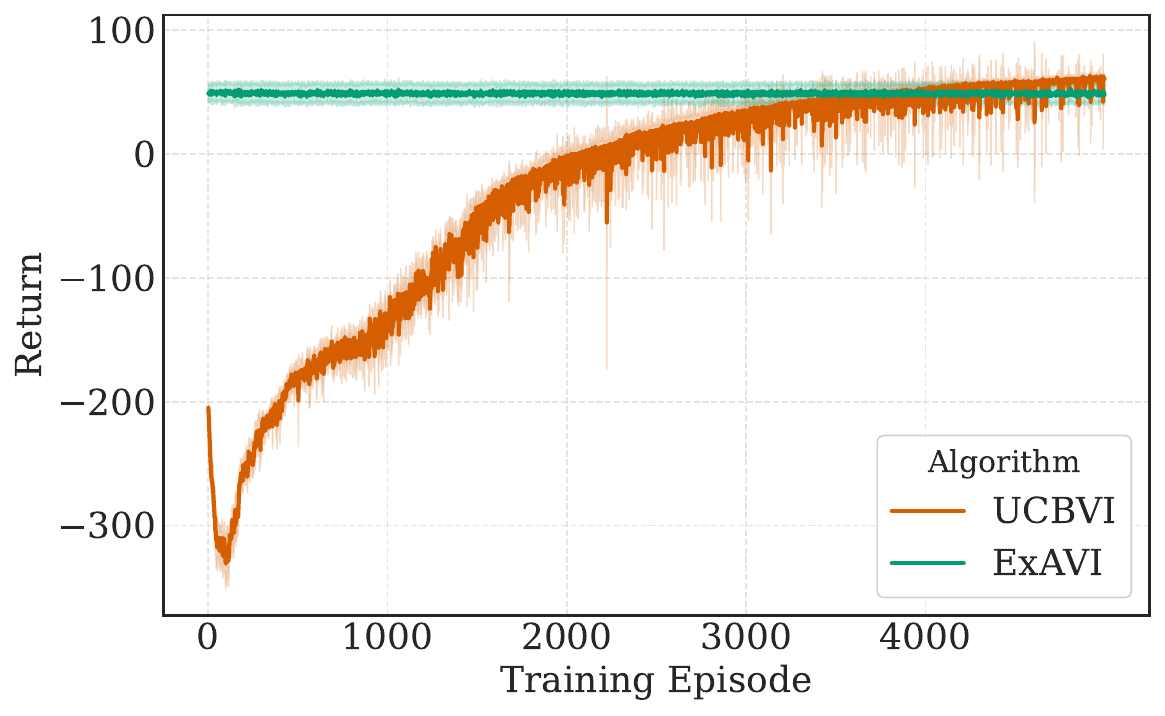}
        \caption{Model-Based Approaches}
        \label{fig:taxi-modelbased}
    \end{subfigure}
    
    
    \begin{subfigure}[b]{\linewidth}
        \centering
        \includegraphics[width=0.86\linewidth]{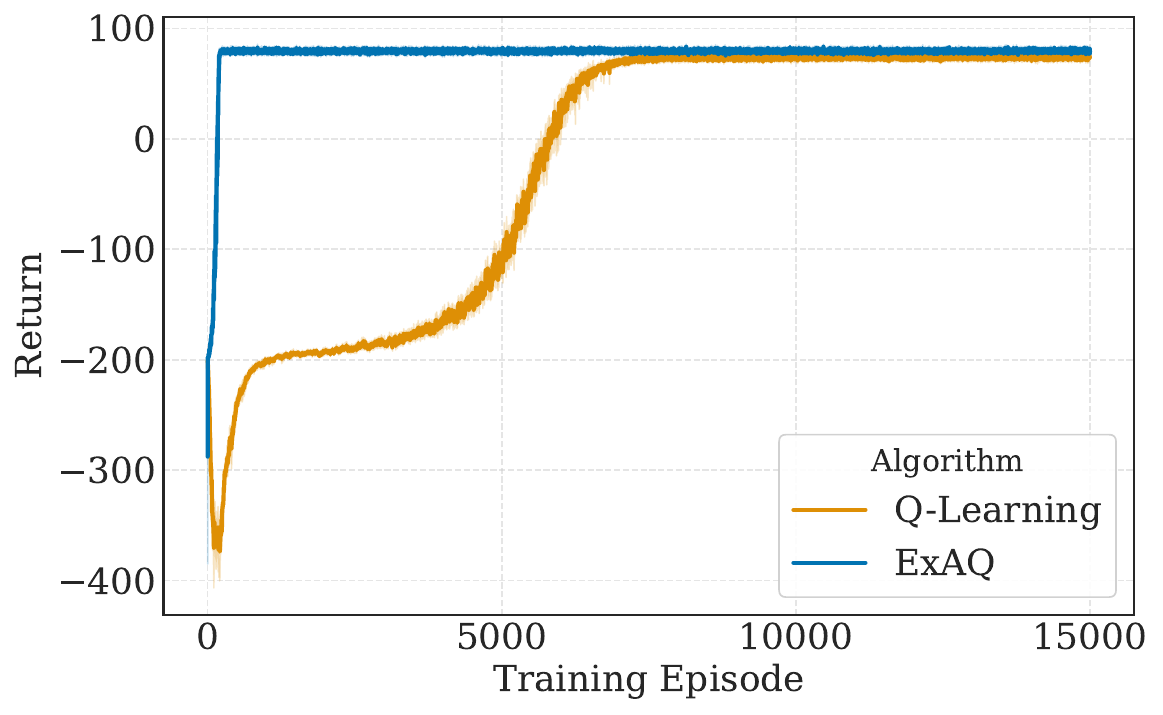}
        \caption{Model-Free Approaches}
        \label{fig:taxi-modelfree}
    \end{subfigure}
    
    \caption{Comparative learning curves for the \texttt{TaxiEnv}, averaged over $10$ training seeds with $95\%$ confidence intervals. Figure (a): Model-based algorithms (\exavi vs \textsc{UCBVI}). Figure (b): Model-free algorithms (\exaq vs \textsc{QL}).}
    \label{fig:taxi-comparison}
\end{figure}

\subsection{Experimental Results}
In this section, we will report the results and performance of the presented algorithms compared to their specific MDP counterparts, focusing on sample efficiency. We will take into account the previously presented environment, leaving the remaining results in Appendix~\ref{app:additional-results}.

\paragraph{Results on \texttt{TaxiEnv}.} In Figure~\ref{fig:taxi-comparison}, we report comparative learning curves on the first scenario for both model-based and model-free approaches. In particular, in Figure~\ref{fig:taxi-modelbased} we compare our model-based method, \exavi, against the \textsc{UCBVI} baseline. As can be noticed, 
\exavi demonstrates significantly superior sample efficiency, converging to the optimal policy almost immediately (within the first few episodes), whereas \textsc{UCBVI} requires thousands of episodes to reach comparable performance. This disparity derives from the structural advantage of the PCMDP formulation: while \textsc{UCBVI} must estimate the full transition kernel $p(\bar s|s,a)$ from scratch, \exavi isolates the estimation problem to the lower-dimensional exogenous component $s^\bullet$. By exploiting the known deterministic transitions of the controllable variables $s^\diamond$, \exavi effectively starts with a partial model of the environment, eliminating the need to explore the endogenous dynamics.

In Figures~\ref{fig:taxi-modelfree}, we report the results for the model-free setting, comparing \exaq against standard \textsc{QL}. Consistent with the model-based case, \exaq exhibits a dramatic improvement in convergence speed. This efficiency gain is attributable to the algorithm's ability to extract the most information from each sample. While \textsc{QL} updates the $Q$-value only for the specific visited state-action pair $(s,a)$, \exaq leverages the independence of the exogenous signal to perform synchronous updates across the entire controllable subspace for the observed exogenous transitions. This effectively acts as a form of \emph{counterfactual reasoning}, allowing the agent to learn the value of unvisited states that share the same exogenous context.
Although \texttt{TaxiEnv} is a simplified benchmark, the empirical results strongly support our theory that explicitly modeling partial controllability significantly improves sample efficiency.

\begin{figure}[t]
    \begin{subfigure}[b]{\linewidth}
        \centering
        \includegraphics[width=0.86\linewidth]{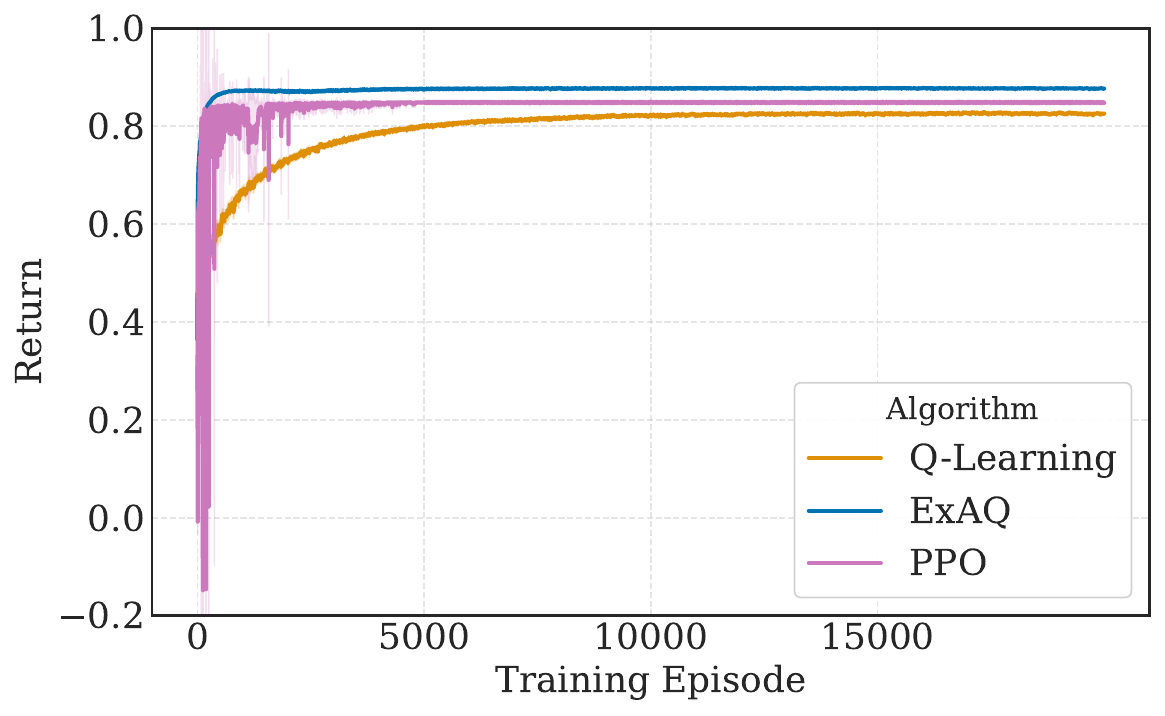}
        \caption{Linear Scale}
        \label{fig:trading}
    \end{subfigure}
    

    \begin{subfigure}[b]{\linewidth}
        \centering
        \includegraphics[width=0.86\linewidth]{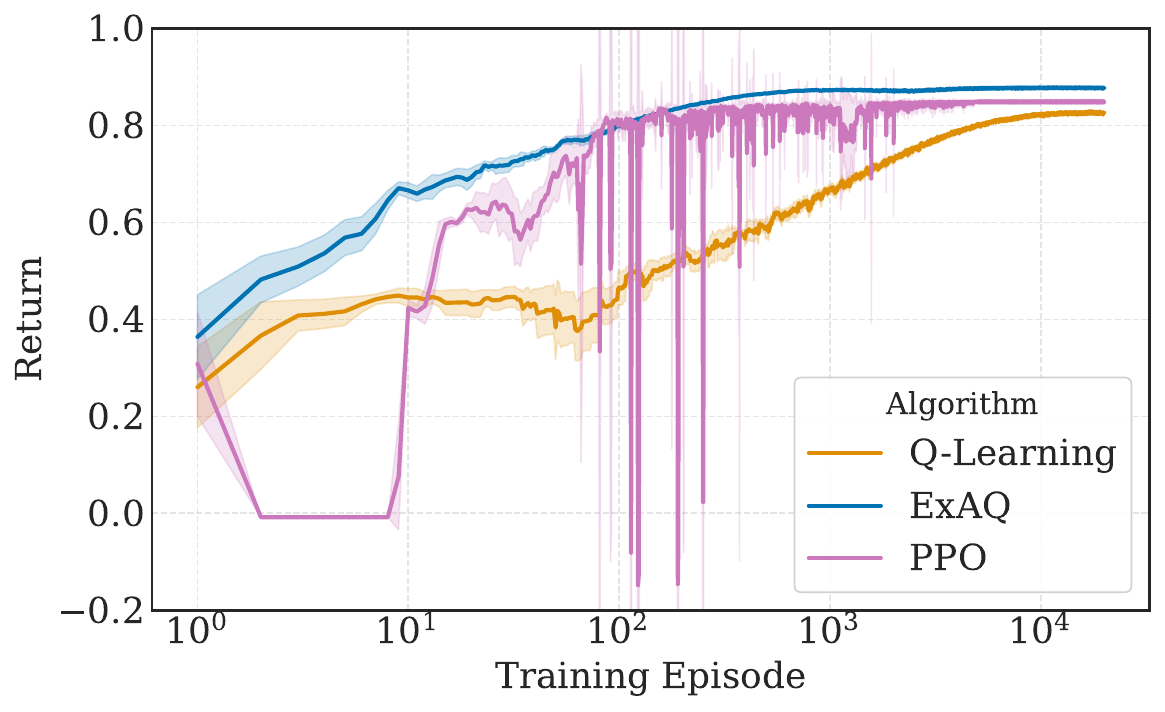}
        \caption{Logarithmic Scale}
        \label{fig:trading-logscale}
    \end{subfigure}

    \caption{Comparative learning curves of model-free algorithms for the \texttt{TradingEnv}, averaged over $10$ training seeds with $95\%$ confidence intervals. Figure (a): Performance on a linear x-axis, highlighting asymptotic convergence. Figure (b): The same experiment plotted on a logarithmic x-axis, highlighting the dramatic sample efficiency gap in the early stages of training ($10^0 - 10^2$ episodes).}
    \label{fig:trading-comparison}
    \label{fig:trading-comparison}
\end{figure}

\paragraph{Results on \texttt{TradingEnv}.} In the following, we present the results obtained on the optimal execution task. For this environment, we restrict our evaluation to model-free approaches, as the high dimensionality of the state space renders tabular model-based planning (e.g., \textsc{UCBVI}) computationally intractable. We compare our PCMDP-based method, \exaq, against two baselines: the standard tabular \textsc{QL} and Proximal Policy Optimization (\textsc{PPO})~\citep{schulman2017proximal}, a state-of-the-art deep RL algorithm frequently applied to optimal execution problems~\citep{lin2021end}.
Figures~\ref{fig:trading-comparison} display the learning curves with the x-axis in linear scale and logarithmic scale. The results underscore the superior sample efficiency of \exaq compared to the baselines. As evident in the logarithmic plot (Figure~\ref{fig:trading-logscale}), \exaq achieves early convergence within the first few orders of magnitude of training episodes ($10^1 - 10^2$). This is remarkable considering the size of the uncontrollable partition of the state space, which is significantly larger than the taxi environment (we recall that $|\mathcal{P}_\omega| = 1000$ compared to $|\boldsymbol{b}^{\text{tr}}| = 2^{N_{\text{tr}}}$ with $N_{\text{tr}}=3$).
In contrast, tabular \textsc{QL} improves slowly and steadily, requiring many more interactions to visit and update the large state-action space. \textsc{PPO} can reach an asymptotic return comparable to \exaq but shows high variance and instability during learning.

\section{Related Works}\label{sec:relatedworks}






While our PCMDP formulation, to our knowledge, is completely novel, other works in the literature have already addressed problems related to partial controllability and the presence of exogenous state variables. For instance, \citet{trimponias2023exogenous} define the concept of exogenous variables within the state space and rewards using the concept of causality combined with the theory of graphs. They propose decomposing the MDP into an exogenous Markov Reward Process (MRP) and an endogenous MDP, optimizing the exogenous and endogenous reward signals, respectively. This allows them to design algorithms that RL methods can leverage to discover the exogenous space, thereby removing the associated reward component.

Some ideas from our work may resonate to other settings that were studied in the classical RL literature.
While \textit{Factored MDPs} \cite{boutilier2000stochastic, guestrin2003efficient} decompose states based on local dependencies, PCMDPs specifically partition them by controllability into endogenous and exogenous components. Unlike the general factorization where actions can affect any variable, PCMDPs assume actions only influence endogenous states, with the agent possessing prior knowledge of these dynamics. This brings heavily improved performance guarantees and, crucially, no need for exploration in algorithm design.

Another class of MDPs that share some aspect with our one is the Post-Decision State (PDS, \cite{powell2007approximate}), where the agent observes a deterministic state following the action before the injection of a random noise.
Both PDS and PCMDPs decompose transitions to isolate agent decisions from environmental noise. However, while PDS introduces a virtual intermediate state affected by subsequent global uncertainty, PCMDPs physically partition the state into independent endogenous and exogenous variables. This decoupling better fits environment like trading, where most of the effort must be put in learning the dynamics of the exogenous part.



In this paper, we also compare with the \textit{online RL} literature, which provides theoretically grounded algorithms both for \textit{model-based}~\cite{azar2017minimax, zhang2024settling, drago2025refined} and \textit{model-free}~\cite{jin2018q} algorithms. While these methods are designed for the classical MDP and do not distinguish between endogenous and exogenous signals, we build on them to obtain PCMDP-based solutions. By separating controllable and uncontrollable transitions, we achieve improved sample efficiency and tighter regret guarantees.

Algorithmic trading is one of the most prominent application domains for RL with exogenous uncertainty. A large body of literature studies RL-based approaches for autonomous trading and optimal execution under stochastic market dynamics \cite{hambly2023recent, ning2021double, bisi2019risk}. These works typically model market variables, such as asset prices, as stochastic processes that evolve independently of the agent’s decisions, while the agent controls internal state variables such as inventory or capital.
While the specific perspective adopted in this paper has not been explicitly formalized in this literature, related ideas appear implicitly in several prior works. In particular, the offline approach proposed by \citet{riva2021learning} leverages historical price trajectories to train a Q-learning–based agent. This methodology inherently assumes that market dynamics evolve exogenously and are independent of the agent’s actions, which aligns with the modeling assumptions underlying our framework.
Our contribution differs by explicitly formalizing this structural independence at the MDP level. We demonstrate how this property can be systematically exploited to derive tighter regret guarantees in an online learning setting. 

\section{Discussion and Conclusions}
In this work, we studied reinforcement learning in Markov Decision Processes with exogenous dynamics (PCMDP) and showed that explicitly distinguishing between controllable and uncontrollable state components leads to substantially improved learning guarantees. Our theoretical results establish regret bounds for both model based and model free algorithms that depend only on the size of the exogenous state space, and we proved that this dependence is information-theoretically optimal. 
We further demonstrated that this framework naturally encompasses a broad class of real-world control problems characterized by action-independent uncertainty, where na\"ive exploration is often costly and unnecessary. To validate the theoretical insights, we applied our approach to the optimal execution problem in algorithmic trading. Our experiments confirm that structure-aware algorithms achieve orders-of-magnitude improvements in sample efficiency compared to standard RL baselines.

\paragraph{Limitations and Future Work.} The primary limitations of this work lie in its restriction to the tabular setting and the assumptions of known controllable dynamics and full observability of the exogenous variables. While the tabular setting provides a rigorous theoretical testbed, the proposed counterfactual updates, although highly sample-efficient, incur a non-trivial computational burden. Future research should address these challenges, extending the framework to the continuous domain via function approximation, which would allow the agent to generalize counterfactual updates across continuous state regions
The assumption of known controllable dynamics can be relaxed by model-based reinforcement learning of $\hat p^\diamond$ alongside $p^\bullet$. Finally, to mitigate the computational complexity, exact planning can be replaced with \emph{Approximate Dynamic Programming} to prioritize updates in the most relevant regions of the state space.\footnote{In Appendix~\ref{app:complexity} we discuss time and space complexity of our methods, comparing them with baseline approaches.}

\section*{Impact Statement}
This paper presents work aimed at advancing the field of machine learning. There are many potential societal consequences of our work, none of which we feel must be specifically highlighted here.

\bibliography{refs}
\bibliographystyle{icml2026}

\newpage
\appendix
\onecolumn

\section{PCMDP in Real-World Scenarios}\label{app:realworld}
Hereafter, we provide additional examples of PCMDP for real-world problems.

\paragraph{Micro-Grid Energy Management.}~~In this task, the agent optimizes energy flow within a micro-grid to maximize revenue from energy trading while minimizing battery degradation. This problem can be viewed as a hybrid of the \emph{water reservoir} (storage management) and \emph{trading} (price arbitrage) tasks.
The system is composed of an energy storage system (ESS), a renewable generator (e.g., PV panels), and a load. The controllable state $s_t^\diamond = (l_t)$ represents the current energy level of the storage system. The exogenous state is defined as $s^\bullet_t = (d_t, g_t, \omega^{\text{buy}}_t, \omega^{\text{sell}}_t)$, where $d_t$ is the power demand, $g_t$ is the renewable generation, and $\omega^{\text{buy}}_t, \omega^{\text{sell}}_t$ denote the grid electricity prices.
The agent selects an action $a_t$ representing the power charged to (positive) or discharged from (negative) the storage. The storage dynamics follow a deterministic accumulation process:
\begin{equation}
    l_{t+1} = l_t + \eta \cdot a_t \Delta t,
\end{equation}
where $\eta$ represents the charging/discharging efficiency and $\Delta t$ the time interval. The exogenous variables evolve according to independent stochastic processes (e.g., weather and market dynamics):
\begin{equation}
    g_{t+1} \sim p(\cdot \mid g_t), \quad d_{t+1} \sim p(\cdot \mid d_t), \quad \omega_{t+1} \sim p(\cdot \mid \omega_t).
\end{equation}
Crucially, while the battery dynamics $l_{t+1}$ are controlled, the exogenous signals dictate the net power exchange with the main grid, defined as $\bar{n}_t = d_t - g_t + a_t$. This value $\bar{n}_t$ determines the immediate reward (cost of purchasing electricity or revenue from selling), making the reward function highly dependent on the exogenous state $s^\bullet_t$ even though the transition $p^\diamond$ is deterministic.

\paragraph{Data Center Thermal Control.}~~Another example is \emph{thermal management in data centers}. The objective is to regulate the cooling infrastructure to maintain server temperatures within safety limits while minimizing cooling energy consumption.
The controllable state $s_t^\diamond = (\mathbf{T}_t^\text{zone})$ is a vector representing the current air temperatures in different server zones. The exogenous state $s_t^\bullet = (\lambda_t, \tau_t^\text{out})$ consists of the current IT workload traffic $\lambda_t$ (which generates heat) and the outside weather conditions $\tau_t^\text{out}$ (which affect cooling efficiency).
The agent controls the airflow or temperature set-points $a_t$. The thermal dynamics are given by standard thermodynamics:
\begin{equation}
    \mathbf{T}_{t+1}^\text{zone} = f(\mathbf{T}_t^\text{zone}, a_t, \lambda_t, \tau_t^\text{out}),
\end{equation}
where the next temperature depends on the current temperature, the cooling action $a_t$, the heat generated by the workload $\lambda_t$, and the ambient leakage $\tau_t^\text{out}$.
This fits the PCMDP framework perfectly: the agent controls the internal temperature $s^\diamond$, but the IT workload $\lambda_t$ (user traffic) and outside weather $\tau_t^\text{out}$ are exogenous processes driven by user behavior and meteorological factors, respectively, which are statistically independent of the cooling system's settings.
\section{Proofs and Regret Analysis}\label{app:proofs}
We dedicate this part of the Appendix to providing formal, extensive proofs of the theorems presented in the main paper.

\subsection{Results of General Interest}

\begin{thm}[Bernstein's inequality \cite{boucheron2003concentration}]\label{thm:bernstein}
    Let $\{X_t\}_{t=1}^n$ be a sequence of zero-mean random variable bounded in $[-B,B]$. Let $\sigma^2:=\sum_{t=1}^n\text{Var}(X_t)$. Then, with probability at least $1-\delta$

    $$\left|\sum_{t=1}^nX_t \right|\le \sqrt{2\sigma^2\log(2/\delta)}+\frac{2B}{3}\log(2/\delta).$$
\end{thm}

\begin{lem}[Lemma 4.1. in \cite{jin2018q}]\label{lem:alphachoice}
    Let $$\alpha_t=\frac{H+1}{H+t}.$$ Then, for any $t\in \N$, $\sum_{i=0}^t \alpha_t^i=1$. Moreover, the following inequalities hold.
    \begin{enumerate}
        \item $\frac{1}{\sqrt t}\le \sum_{i=1}^t\frac{\alpha_t^i}{\sqrt i}\le \frac{2}{\sqrt t}$
        \item $\max_i \alpha_t^i\le \frac{2H}{t}$
        \item $\sum_{i=1}^t(\alpha_i^t)^2\le \frac{2H}{t}$
        \item $\sum_{t=i}^\infty\alpha_i^t= 1+\frac{1}{H}$
    \end{enumerate}
\end{lem}
\subsection{Regret Bound for \exavi}\label{app:proof-exavi}
To prove the main theorem concerning the regret bound of algorithm \ref{alg:exavi}, we start from a result about the confidence bound of $\widehat p_h^{\bullet,k}$.

\begin{thm}
    Fix $\delta>0$. With probability at least $1-\delta$, at the same time for all $\bar s^{\bullet},s^\bullet\in \Ss^\bullet$ we have
    $$|\widehat p_h^{\bullet,k}(\bar s^{\bullet}|s^\bullet)-p_h^\bullet(\bar s^{\bullet}|s^\bullet)|\le \sqrt{\frac{2p_h^\bullet(\bar s^{\bullet}|s^\bullet)(1-p_h^\bullet(\bar s^{\bullet}|s^\bullet))\log(2KS^\bullet/\delta)}{n_h^k(s^\bullet)}}+\frac{4\log(2KS^\bullet/\delta)}{3n_h^k(s^\bullet)}$$
\end{thm}
\begin{proof}
    Fix $\bar s^{\bullet},s^\bullet\in \Ss^\bullet$, and let $t=n_h^k(s^\bullet)$ a given deterministic value. The variable
    $\frac{1}{t}\sum_{i=1}^t 1\{s_{h+1}^{\bullet ,k_i}=\bar s^{\bullet}\}$ defined by equation \eqref{eq:phat} is the average of a martingale difference sequence. The variables $X_i:=1\{s_{h+1}^{\bullet ,k_i}=\bar s^{\bullet}\}$ we are summing satisfy $\sigma^2:=\sum_{i=1}^t \textsc{Var}(X_i)=tp_h^\bullet(\bar s^{\bullet}|s^\bullet)(1-p_h^\bullet(\bar s^{\bullet}|s^\bullet)).$ Moreover, all r.v. are bounded by $1$. Therefore, by Bernstein's inequality \ref{thm:bernstein},

    \begin{align*}
        |\widehat p_h^{\bullet,k}(\bar s^{\bullet}|s^\bullet)-p_h^\bullet(\bar s^{\bullet}|s^\bullet)|&=\left|t^{-1}\sum_{i=1}^t X_i-p_h(\bar s^{\bullet}|s^\bullet)\right|\\
        &\le \sqrt{\frac{p_h^\bullet(\bar s^{\bullet}|s^\bullet)(1-p_h^\bullet(\bar s^{\bullet}|s^\bullet))\log(2/\delta)}{t}}+\frac{2\log(2/\delta)}{3t}.
    \end{align*}

    Making a union bound over the possible values of $\bar s^{\bullet},s^\bullet$ and $t\le k\le K$ (that is, dividing $\delta$ by $K{S^\bullet}^2$) ends the proof.
\end{proof}

\begin{prop}\label{prop:unip}
    Fix $\delta>0$. For every $f:\Ss^\bullet\to [0,1]$ and every $s^{\bullet}\in \Ss^\bullet$ we have, with probability at least $1-\delta$,

    $$\left|\sum_{\bar s^{\bullet}\in \Ss^\bullet}f(\bar s^{\bullet})(\widehat p_h^{\bullet,k}(\bar s^{\bullet}|s^\bullet)-p_h^\bullet(\bar s^{\bullet}|s^\bullet))\right|\le \frac{3HS^\bullet\log(2KS^\bullet/\delta)}{n_h^k(s^\bullet)}+H^{-1}\sum_{\bar s^{\bullet}\in \Ss^\bullet}f(\bar s^{\bullet})p_h^\bullet(\bar s^{\bullet}|s^\bullet).$$
\end{prop}
\begin{proof}
    Let $t:=n_h^k(s^\bullet)$. By the previous theorem, we have
    \begin{align*}
        \left|\sum_{\bar s^{\bullet}\in \Ss^\bullet}f(\bar s^{\bullet})(\widehat p_h^{\bullet,k}(\bar s^{\bullet}|s^\bullet)-p_h^\bullet(\bar s^{\bullet}|s^\bullet))\right|&\le \sum_{\bar s^{\bullet}\in \Ss^\bullet}f(\bar s^{\bullet})\left|\widehat p_h^{\bullet,k}(\bar s^{\bullet}|s^\bullet)-p_h^\bullet(\bar s^{\bullet}|s^\bullet)\right|\\
        &\le \sum_{\bar s^{\bullet}\in \Ss^\bullet}f(\bar s^{\bullet})\sqrt{\frac{2p_h^\bullet(\bar s^{\bullet}|s^\bullet)(1-p_h^\bullet(\bar s^{\bullet}|s^\bullet))\log(2KS^\bullet/\delta)}{t}}\\
        &\quad +\sum_{\bar s^{\bullet}\in \Ss^\bullet}f(\bar s^{\bullet})\frac{4\log(2KS^\bullet/\delta)}{3t}.
    \end{align*}

    We start bounding the first term:

    \begin{align*}
        \text{First term} &=\sum_{\bar s^{\bullet}\in \Ss^\bullet}\sqrt{\frac{2f(\bar s^{\bullet})^2p_h^\bullet(\bar s^{\bullet}|s^\bullet)(1-p_h^\bullet(\bar s^{\bullet}|s^\bullet))\log(2KS^\bullet/\delta)}{t}}\\
        &\le \sum_{\bar s^{\bullet}\in \Ss^\bullet}\sqrt{\frac{2f(\bar s^{\bullet})^2p_h^\bullet(\bar s^{\bullet}|s^\bullet)\log(2KS^\bullet/\delta)}{t}}\\
        &\le \sqrt{2S^\bullet\sum_{\bar s^{\bullet}\in \Ss^\bullet}\frac{f(\bar s^{\bullet})^2p_h^\bullet(\bar s^{\bullet}|s^\bullet)\log(2KS^\bullet/\delta)}{t}}\\
        &\overset{(*)}{\le} \sqrt{\left(HS^\bullet\frac{\log(2KS^\bullet/\delta)}{t}\right)^2+\left(H^{-1}\sum_{\bar s^{\bullet}\in \Ss^\bullet}f(\bar s^{\bullet})^2p_h^\bullet(\bar s^{\bullet}|s^\bullet)\right)^2}\\
        &\overset{(**)}{\le}\frac{HS^\bullet\log(2KS^\bullet/\delta)}{t}+H^{-1}\sum_{\bar s^{\bullet}\in \Ss^\bullet}f(\bar s^{\bullet})^2p_h^\bullet(\bar s^{\bullet}|s^\bullet).
    \end{align*}

    Here, (*) follows from $ab\le a^2+b^2$, and (**) from $\sqrt{a+b}\le \sqrt a+\sqrt b$. As $f$ takes values in $[0,1]$, we also have

    $$\text{First term}\le \frac{HS^\bullet\log(2KS^\bullet/\delta)}{t}+H^{-1}\sum_{\bar s^{\bullet}\in \Ss^\bullet}f(\bar s^{\bullet})^2p_h^\bullet(\bar s^{\bullet}|s^\bullet).$$

    For what concerns the second term, we have

    $$\sum_{\bar s^{\bullet}\in \Ss^\bullet}f(\bar s^{\bullet})\frac{4\log(2KS^\bullet/\delta)}{3t}\le \frac{4S^\bullet\log(2KS^\bullet/\delta)}{3t}.$$

    This completes the proof.
\end{proof}

After the confidence bounds, we move to some theorems about the convergence of the empirical Bellman operator.
For the next steps, it is useful to define the following operators for a function $f:\Ss\to \R$.

$$P_h[f](s,a)=\sum_{\bar s^{\bullet}\in \Ss^\bullet}\widehat p_h^\bullet(\bar s^{\bullet}|s^\bullet)\sum_{\bar s^{\diamond}\in \Ss^\diamond}p_h^\diamond(\bar s^{\diamond},a)f(\bar s^{\bullet},\bar s^{\diamond}),$$

and its empirical counterpart

$$\widehat P_h^k[f](s,a)=\sum_{\bar s^{\bullet}\in \Ss^\bullet}\widehat p_h^{\bullet,k}(\bar s^{\bullet}|s^\bullet)\sum_{\bar s^{\diamond}\in \Ss^\diamond}p_h^\diamond(\bar s^{\diamond},a)f(\bar s^{\bullet},\bar s^{\diamond}).$$

\begin{lem}\label{lem:anyf}
    Fix $\delta>0$. For every $f:\Ss^\bullet\to [0,1]$ and every $s\in \Ss,a\in \As$ we have, with probability at least $1-\delta$,

    $$\left|P_h[f](s,a)-\widehat P_h^k[f](s,a)\right|\le \frac{6HS^\bullet\log(2KSA/\delta)}{n_h^k(s^\bullet)}+H^{-1}P_h[f](s,a).$$
\end{lem}
\begin{proof}
    Call
    $$\Phi[f](\bar s^\bullet; s,a):=\sum_{\bar s^\diamond\in \Ss^\diamond}f(\bar s^\bullet,\bar s^\diamond)p_h^\diamond(\bar s^\diamond |s,a).$$

    Note that in the previous equation $p_h^\diamond$ is known, so the full operator can be computed by the learner.
    
    \begin{align*}
        P_h[f](s,a)&=\sum_{\bar s^\bullet\in \Ss^\bullet}\sum_{\bar s^\diamond\in \Ss^\diamond}f(\bar s^\bullet,\bar s^\diamond)p_h(\bar s^\bullet,\bar s^\diamond|s,a)\\
        &=\sum_{\bar s^\bullet\in \Ss^\bullet}\sum_{\bar s^\diamond\in \Ss^\diamond}f(\bar s^\bullet,\bar s^\diamond)p_h^\bullet(\bar s^\bullet|s^\bullet)p_h^\diamond(\bar s^\diamond |s,a)\\
        &=\sum_{\bar s^\bullet\in \Ss^\bullet}p_h^\bullet(\bar s^\bullet|s^\bullet)\underbrace{\sum_{\bar s^\diamond\in \Ss^\diamond}f(\bar s^\bullet,\bar s^\diamond)p_h^\diamond(\bar s^\diamond |s,a)}_{\Phi[f](\bar s^\bullet; s,a)}.
    \end{align*}

    The same passages performed on $\widehat P_h^k$ give

    $$\widehat P_h^k[f](s,a)=\sum_{\bar s^\bullet\in \Ss^\bullet}\widehat p_h^{\bullet,k}(\bar s^\bullet|s^\bullet)\Phi[f](\bar s^\bullet; s,a).$$

    Therefore, we have, for fixed $s,a$, and $t:=n_h^k(s^\bullet)$,

    \begin{align*}
        |\widehat P_h^k[f](s,a)-P_h[f](s,a)| &=\left|\sum_{\bar s^\bullet\in \Ss^\bullet}(p_h^\bullet(\bar s^\bullet|s^\bullet)-\widehat p_h^{\bullet,k}(\bar s^\bullet|s^\bullet))\Phi[f](\bar s^\bullet; s,a)\right| \\
        &\le \frac{3HS^\bullet\log(2KS^\bullet/\delta)}{t}+H^{-1}\sum_{\bar s^{\bullet}\in \Ss^\bullet}\Phi[f](\bar s^\bullet; s,a)p_h^\bullet(\bar s^{\bullet}|s^\bullet)\\
        &\le \frac{3HS^\bullet\log(2KS^\bullet/\delta)}{t}+H^{-1}P_h[f](s,a),
    \end{align*}

    where the inequality comes from proposition \ref{prop:unip}. Making a union bound over $\Ss,\As$ gives the statement.
\end{proof}

We need one last result before proceeding to the regret bound.
While the following lemma looks similar to the previous one, here we are fixing $f$ in advance, while before the probability $1-\delta$ was for all the values of $f$ at the same time.

\begin{lem}\label{lem:fixedf}
    For a function $f:\Ss^\bullet\to [0,1]$ and $\delta>0$. For every $s\in \Ss,a\in \As$ we have, with probability at least $1-\delta$,

    $$\left|P_h[f](s,a)-\widehat P_h^k[f](s,a)\right|\le \sqrt{\frac{\log(2SA/\delta)}{2n_h^k(s^\bullet)}}.$$
\end{lem}
\begin{proof}
    As before, we write

    $$P_h[f](s,a)=\sum_{\bar s^\bullet\in \Ss^\bullet}p_h^{\bullet}(\bar s^\bullet|s^\bullet)\Phi[f](\bar s^\bullet; s,a)$$

    and

    $$\widehat P_h^k[f](s,a)=\sum_{\bar s^\bullet\in \Ss^\bullet}\widehat p_h^{\bullet,k}(\bar s^\bullet|s^\bullet)\Phi[f](\bar s^\bullet; s,a).$$

    Therefore, we have, for fixed $s,a$,

    \begin{align*}
        |\widehat P_h^k[f](s,a)-P_h[f](s,a)| &=\left|\sum_{\bar s^\bullet\in \Ss^\bullet}(p_h^\bullet(\bar s^\bullet|s^\bullet)-\widehat p_h^{\bullet,k}(\bar s^\bullet|s^\bullet))\Phi[f](\bar s^\bullet; s,a)\right|\\
        &=\left|\sum_{\bar s^\bullet\in \Ss^\bullet}\left(p_h^\bullet(\bar s^\bullet|s^\bullet)-\frac{1}{t}\sum_{i=1}^t 1\{s_{h+1}^{\bullet ,k_i}=\bar s^{\bullet}\}\right)\Phi[f](\bar s^\bullet; s,a)\right|\\
        &=\left|\sum_{\bar s^\bullet\in \Ss^\bullet}p_h^\bullet(\bar s^\bullet|s^\bullet)\Phi[f](\bar s^\bullet; s,a)-\frac{1}{t}\Phi[f](s_{h+1}^{\bullet ,k_i}; s,a)\right|.
    \end{align*}

    Let us fix $t$ for now. In this situation, as $s_{h+1}^{\bullet ,k_i}$ is sampled from $p_h^\bullet(\cdot|s^\bullet)$ and $f,s,a$ are fixed, we can apply Hoeffding's inequality, which gives, w.p. at least $1-\delta$,

    \begin{align*}
        |\widehat P_h^k[f](s,a)-P_h[f](s,a)|&=\left|\sum_{\bar s^\bullet\in \Ss^\bullet}p_h^\bullet(\bar s^\bullet|s^\bullet)\Phi[f](\bar s^\bullet; s,a)-\frac{1}{t}\Phi[f](s_{h+1}^{\bullet ,k_i}; s,a)\right|\\
        &\le \sqrt{\frac{\log(2/\delta)}{2t}}.
    \end{align*}

    Recalling that is actually a random variable $t=n_h^k(s^\bullet)\le k$, we can make a union bound over the number $k$ of possible values and write

    \begin{align*}
        |\widehat P_h^k[f](s,a)-P_h[f](s,a)|\le \sqrt{\frac{\log(2/\delta)}{2n_h^k(s^\bullet)}}.
    \end{align*}

    Making a union bound over $\Ss,\As$ ends the proof.
\end{proof}

\subsubsection{Proofs of convergence of the estimated value functions}

In this section, we are going to use the shorthands

$$\xi_h^k(s^\bullet):=\frac{6H^2S^\bullet\log(2KSA/\delta)}{n_h^k(s^\bullet)}+\sqrt{\frac{H^2\log(2SA/\delta)}{2n_h^k(s^\bullet)}}\qquad \widehat V_1^{\pi}:=\widehat V_1^{\pi,k}.$$

\begin{prop}\label{prop:optimalvaldiffer}
    With probability at least $1-\delta$,
    \begin{align*}
        V_1^{\star}(s)-\widehat V_1^{\pi^\star}(s)&\le e\E\left[\sum_{h=1}^{H}\xi_h^k(s_h^\bullet)\bigg |s_1^\bullet=s^\bullet\right].
    \end{align*}
\end{prop}
\begin{proof}
    At any $h$, it holds

    \begin{align*}
        V_h^{\star}(s)-\widehat V_h^{\pi^\star}(s) &=P_h[V_{h+1}^{\star}](s,\pi_h^\star(s))-\widehat P_h^k[\widehat V_{h+1}^{\pi^\star}](s,\pi_h^\star(s))\\
        &=P_h[V_{h+1}^{\star}](s,\pi_h^\star(s))-\widehat P_h^k[V_{h+1}^{\star}](s,\pi_h^\star(s))\\
        &\quad + \widehat P_h^k[V_{h+1}^{\star}-\widehat V_{h+1}^{\pi^\star}](s,\pi_h^\star(s))\\
        &=\underbrace{(\widehat P_h^k-P_h)[V_{h+1}^{\star}-\widehat V_{h+1}^{\pi^\star}](s,\pi_h^\star(s))}_{T1}\\
        &\quad + \underbrace{P_h[V_{h+1}^{\star}](s,\pi_h^\star(s))-\widehat P_h^k[V_{h+1}^{\star}](s,\pi_h^\star(s))}_{T2}\\
        &\quad +P_h[V_{h+1}^{\star}-\widehat V_{h+1}^{\pi^\star}](s,\pi_h^\star(s)).
    \end{align*}

    We can bound $T1$ and $T2$ as follows.

    \begin{enumerate}
        \item By lemma \ref{lem:anyf}, for $f=V_{h+1}^{\star}-\widehat V_{h+1}^{\pi^\star}$ (this function has an output in $[-H,H]$ instead of $[0,1]$, so we have to add a $2H$ constant to the first term), $s=s$ and $a=\pi_h^\star(s)$,
        \begin{align*}
            T1 &=\widehat P_h^k[f](s,a)-P_h[f](s,a)\\
            &\le 
            \frac{6H^2S^\bullet\log(2KSA/\delta)}{t}+H^{-1}P_h[V_{h+1}^{\star}-\widehat V_{h+1}^{\pi^\star}](s,a).
        \end{align*}

        \item By lemma \ref{lem:fixedf}, for $f=V_{h+1}^{*}$, which is a deterministic function (even if bounded in $[0,H]$ instead of $[0,1]$, which forces to add an $H$ factor), and $s=s,a=\pi_h^k(s)$
        \begin{align*}
            T2 &= (\widehat P_h^k-P_h)[V_{h+1}^\star](s,\pi_h^\star(s))\le H\sqrt{\frac{\log(2SA/\delta)}{2t}}.
        \end{align*}
    \end{enumerate}

    Replacing $t=n_h^k(s^\bullet)$ we have

    $$V_h^{\star}(s)-\widehat V_h^{\pi^\star}(s)\le \underbrace{\frac{6H^2S^\bullet\log(2KSA/\delta)}{n_h^k(s^\bullet)}+\sqrt{\frac{H^2\log(2SA/\delta)}{2n_h^k(s^\bullet)}}}_{\xi_h^k(s^\bullet)}+(1+H^{-1})P_h[V_{h+1}^{\star}-\widehat V_{h+1}^{\pi^\star}](s,\pi_h^\star(s)).$$

    This implies, by recursion

    \begin{align*}
        V_h^{\star}(s)-\widehat V_h^{\pi^\star}(s)&\le \E\left[\sum_{\tau=h}^{H}\left(1+1/H\right)^{\tau-h}\xi_\tau^k(s_\tau^\bullet)\bigg |s_h=s, \pi^*\right]\\
        &\le e\E\left[\sum_{\tau=h}^{H}\xi_\tau^k(s_\tau^\bullet)\bigg |s_h^\bullet=s^\bullet\right].
    \end{align*}

    In the special case $h=1$, this is the thesis. Note that at the last step we have erased the dependence on $\pi^\star$ and $s^\diamond$, as from the structure of the process, the number of times $s^\bullet$ is visited does not depend on the actions.
\end{proof}

A similar result holds for the estimated optimal policy.

\begin{prop}\label{prop:kvaldiffer}
    The difference between the estimated value function of policy $\pi^k$ and its true value satisfies 

    $$\widehat V_1^{k}(s_0)-V_1^{\pi_k}(s_0)\le (1+e)\E\left[\sum_{h=1}^{H}\xi_h^k(s_h^\bullet)\bigg |s_1^\bullet=s^\bullet\right].$$
\end{prop}
\begin{proof}
    By the simulation lemma,
    \begin{align}
        \widehat V_1^k(s_0)-V_1^{\pi_k}(s_0) &=\sum_{h=1}^H\E_{s_h,a_h\sim d_h^{\pi_h^k}}\left[(\widehat P_h^k-P_h)[\widehat V_{h+1}^k](s_h,\pi_h^k(s_h))\right] \label{eq:vhatk}\\
        &=\sum_{h=1}^H\E_{s_h,a_h\sim d_h^{\pi_h^k}}\left[(\widehat P_h^k-P_h)[\widehat V_{h+1}^k-V_{h+1}^\star](s_h,\pi_h^k(s_h))\right]\label{eq:firstpart}\\
        &\qquad +\sum_{h=1}^H\underbrace{\E_{s_h,a_h\sim d_h^{\pi_h^k}}\left[(\widehat P_h^k-P_h)[V_{h+1}^\star](s_h,\pi_h^k(s_h))\right]}_{\rho_h^k}\label{eq:rhohdef},
    \end{align}

    where the last line follows by subtracting and adding $(\widehat P_h^k-P_h)[V_h^*](s_h,\pi_h^k(s_h))$.
    By lemma \ref{lem:fixedf}, for $f=V_{h+1}^\star$, which is a deterministic function (even if bounded in $[0,H]$ instead of $[0,1]$, which forces to add an $H$ factor), and $s=s,a=\pi_h^k(s)$, the inner term of the second part is bounded by
    \begin{align}
        \rho_h^k &= (\widehat P_h^k-P_h)[V_{h+1}^\star](s,\pi_h^k(s))\\
        &=\left|\widehat P_h^k[f](s,a)-P_h[f](s,a)\right|\le H\sqrt{\frac{\log(2SA/\delta)}{2n_h^k(s^\bullet)}}\label{eq:rhoh}.
    \end{align}
    
    We now bound the first one. By definition of $V_h^\star$ and $\widehat V_h^k$ we have, for every $h\in [H]$ and $s'\in \Ss$,

    $\widehat V_h^k(s)-V_h^\star(s)=\max_{a\in \As} r_h(s,a)+\widehat P_h^k[\widehat V_{h+1}^k](s,a)-\max_{a\in \As} r_h(s,a)+P_h[V_{h+1}^\star](s,a)$. Therefore,

    \begin{align*}
        \widehat V_h^k(s)-V_h^\star(s) &\le \widehat V_h^k(s)-r_h(s,\pi_h^k(s))-\sum_{\bar s\in \Ss}p_h(\bar s|s,\pi_h^k(s))V_h^\star(\bar s)\\
        &=\widehat P_h^k[\widehat V_{h+1}^k](s,\pi_h^k(s))-P_h[V_{h+1}^\star](s,\pi_h^k(s))\\
        &=\widehat P_h^k[\widehat V_{h+1}^k](s,\pi_h^k(s))- P_h[\widehat V_{h+1}^k](s,\pi_h^k(s))\\
        &\ \ +P_h[\widehat V_{h+1}^k](s,\pi_h^k(s))-P_h[V_{h+1}^\star](s,\pi_h^k(s))\\
        &=(\widehat P_h^k-P_h)[\widehat V_{h+1}^k](s,\pi_h^k(s))\\
        &\ \ +P_h[\widehat V_{h+1}^k-V_{h+1}^\star](s,\pi_h^k(s))\\
        &=\underbrace{(\widehat P_h^k-P_h)[\widehat V_{h+1}^k-V_{h+1}^\star](s,\pi_h^k(s))}_{T1}\\
        &\ \ +\underbrace{(\widehat P_h^k-P_h)[V_{h+1}^\star](s,\pi_h^k(s))}_{T2}\\
        &\ \ +P_h[\widehat V_{h+1}^k-V_{h+1}^\star](s,\pi_h^k(s)).
    \end{align*}

    We can bound $T1$ and $T2$:

    \begin{enumerate}
        \item By lemma \ref{lem:anyf}, for $f=\widehat V_{h+1}^k-V_{h+1}^\star$ (this function has an output in $[-H,H]$ instead of $[0,1]$, so we have to add a $2H$ constant to the first term), $s=s$ and $a=\pi_h^k(s)$,
        \begin{align}
            T1 &=\widehat P_h^k[f](s,a)-P_h[f](s,a)\\
            &\le \frac{6H^2S^\bullet\log(2KSA/\delta)}{t}+H^{-1}P_h[\widehat V_{h+1}^{k}-V_{h+1}^{\star}](s,a).\label{eq:stragerico}
        \end{align}

        \item By lemma \ref{lem:fixedf}, for $f=V_{h+1}^\star$, which is a deterministic function (even if bounded in $[0,H]$ instead of $[0,1]$, which forces to add an $H$ factor), and $s=s,a=\pi_h^k(s)$
        \begin{align*}
            T2 &= (\widehat P_h^k-P_h)[V_{h+1}^\star](s,\pi_h^k(s))\le H\sqrt{\frac{\log(2SA/\delta)}{2t}}.
        \end{align*}
    \end{enumerate}

    Therefore, we get

    $$\widehat V_h^k(s)-V_h^\star(s)\le (1+H^{-1})P_h[\widehat V_{h+1}^k-V_{h+1}^\star](s,\pi_h^k(s))+\xi_h^k(s^\bullet),$$

    which, as seen in the proof of proposition \ref{prop:optimalvaldiffer}, implies

    \begin{align*}
        V_h^{k}(s)-\widehat V_h^\star(s)\le e\E\left[\sum_{\tau=h}^{H}\xi_\tau^k(s_\tau^\bullet)\bigg |s_h=s\right].
    \end{align*}

    Replacing the result of equation \eqref{eq:stragerico} back into the initial problem, we get

    \begin{align*}
        &\sum_{h=1}^H\E_{s_h,a_h\sim d_h^{\pi_h^k}}\left[(\widehat P_h^k-P_h)[\widehat V_{h+1}^k-V_{h+1}^*](s_h,\pi_h^k(s_h))\right]\\
        &\qquad \overset{\ref{eq:stragerico}}{\le} \sum_{h=1}^H \E_{s_h,a_h\sim d_h^{\pi_h^k}}\left[\frac{6H^2S^\bullet\log(2KSA/\delta)}{n_h^k(s_h^\bullet)}+H^{-1}e\E\left[\sum_{\tau=h+1}^{H}\xi_\tau^k(s_\tau)\bigg |s_h=s\right]\right]\\
        &\qquad \le \sum_{h=1}^H \E_{s_h,a_h\sim d_h^{\pi_h^k}}\left[\xi_h^k(s_h^\bullet)+H^{-1}e\E\left[\sum_{\tau=h+1}^{H}\xi_\tau^k(s_\tau)\bigg |s_h=s\right]\right]\\
        &\qquad \le e\sum_{h=1}^H \E_{s_h,a_h\sim d_h^{\pi_h^k}}\left[\xi_h^k(s_h^\bullet)\right]\\
        &\le \sum_{h=1}^H \E_{s_h,a_h\sim d_h^{\pi_h^k}}\left[\xi_h^k(s_h^\bullet)\right]=e\E\left[\sum_{h=1}^{H}\xi_h^k(s_h^\bullet)\right].
    \end{align*}

    We use these passages to upper bound the first term in equation \ref{eq:firstpart}. We have already proved (equation \ref{eq:rhoh}) that the other term is bounded by
    $$\sum_{h=1}^H \E_{s_h,a_h\sim d_h^{\pi_h^k}}\left[\sqrt{\frac{H^2\log(2SA/\delta)}{2n_h^k(s_h^\bullet)}}\right]\le  \sum_{h=1}^H \E_{s_h,a_h\sim d_h^{\pi_h^k}}\left[\xi_h^k(s_h^\bullet)\right]=\E\left[\sum_{h=1}^{H}\xi_h^k(s_h^\bullet)\right].$$

    The last passage holds since the trajectory $\{s_h^\bullet\}_h$ is not influenced by the actions.
    Having bound both parts of equation \ref{eq:vhatk}, this completes the proof.
    
\end{proof}

\subsubsection{Regret bound}

\begin{thm}
    With probability at least $1-\delta$, the regret of algorithm \ref{alg:exavi} is bounded by
    $$R_K\le \bigot\left(H\sqrt{S^\bullet K\log(1/\delta)}+H^2{S^\bullet}^2\log(1/\delta)\right).$$
\end{thm}
\begin{proof}
    By design of the algorithm, we know that the policy $\pi_k$ we choose at each time step is such that
    $$\widehat V_1^{\pi_k,k}(s_0)=\widehat V_1^k(s_0)\ge \widehat V_1^{\pi^\star,k}(s_0),$$

    where $\widehat V_1^{\pi,k}(s_0)$ is the value function of a given policy if we replace the true transition $\{p_h\}_{h=1}^H$ with its estimation $\{\widehat p_h^k\}_{h=1}^H$ at episode $k$. Therefore, by the definition of regret,

    \begin{align*}
        R_K&=\sum_{k=1}^KV_1^{\star}(s_0)-V_1^{\pi_k}(s_0)\\
        &=\sum_{k=1}^KV_1^{\star}(s_0)-\widehat V_1^k(s_0)+\widehat V_1^k(s_0)-V_1^{\pi_k}(s_0)\\
        &=\sum_{k=1}^K\underbrace{V_1^{\star}(s_0)-\widehat V_1^{\pi^\star}(s_0)}_{\eqref{prop:optimalvaldiffer}}+\underbrace{\widehat V_1^{\pi^\star}(s_0)-\widehat V_1^k(s_0)}_{\le 0}+\underbrace{\widehat V_1^k(s_0)-V_1^{\pi_k}(s_0)}_{\eqref{prop:kvaldiffer}}\\
        &\le \sum_{k=1}^K \left((1+2e)\E\left[\sum_{h=1}^{H}\xi_h^k(s_h^{\bullet,k})\right]\right),
    \end{align*}

    where $\xi_h^k(s^\bullet)=\frac{6H^2S^\bullet\log(2KSA/\delta)}{n_h^k(s^\bullet)}+\sqrt{\frac{H^2\log(2SA/\delta)}{2n_h^k(s^\bullet)}}$. 
    Now, we only need to bound this quantity for any state-action sequence. Indeed, using the standard pigeonhole argument,

    \begin{align*}
        \sum_{k=1}^K\sum_{h=1}^{H}\xi_h^k(s_h^{\bullet,k})&=\sum_{k=1}^K\sum_{h=1}^{H}\frac{C_1}{n_h^k(s_h^{\bullet,k})}+\frac{C_2}{\sqrt{n_h^k(s_h^{\bullet,k})}}\\
        &=\sum_{h=1}^{H}\sum_{k=1}^K\frac{C_1}{n_h^k(s_h^{\bullet,k})}+\frac{C_2}{\sqrt{n_h^k(s_h^{\bullet,k})}}\\
        &=\sum_{h=1}^{H}\sum_{s^\bullet \in \Ss^\bullet}\sum_{i=1}^{n_h^K(s^\bullet)}\frac{C_1}{i}+\frac{C_2}{\sqrt{i}}\\
        &=\sum_{h=1}^{H}\sum_{s^\bullet \in \Ss^\bullet}C_1\log(n_h^K(s^\bullet))+2C_2\sqrt{n_h^K(s^\bullet)}\\
        &\le HC_1S^\bullet \log(K)+2HC_2\sqrt{S^\bullet K},
    \end{align*}

    where $C_1=6H^2S^\bullet\log(2KSA/\delta)$ and $C_2=\sqrt{H^2\log(2SA/\delta)/2}$. This completes the proof.
\end{proof}
\subsection{Regret Bound for \exaq}\label{app:proof-exaq}
In this section, we prove the regret bound for our value-based algorithm \ref{alg:exaq}. We start by recalling some definitions that help in lightening the notation.
Define, for $f:\Ss\times \As\to \R$ and $g:\Ss^\bullet\to \R$ the following operators:

$$\textbf{P}_h^{k}[f](s,a):=\sum_{\bar s^\diamond\in \Ss^\diamond}f(\bar s^\diamond,s_{h+1}^{\bullet,k})p_h^\diamond(\bar s^\diamond|s,a)\qquad \textbf{P}_h^{\bullet,k}[g](s^\bullet):=g(s_{h+1}^{\bullet,k}).$$

Fix $(s,a)$ and call $k_i$ the episodes where $s_h^{\bullet,k_i}=s^\bullet$, so that $t=n_h^k(s^\bullet)$.

\begin{lem}
    Let $h\le H-1$ and $\alpha_t^i$ be a learning rate schedule such that $\sum_{i=0}^t \alpha_t^i=1$. The following equation holds
    $$Q_h^k(s,a)-Q_h^\star(s,a)=\alpha_t^0(H-Q_h^\star(s,a))+\sum_{i=1}^t\alpha_t^i \Big(\textbf{P}_h^{k_i}[V_{h+1}^{k_i}-V_{h+1}^\star](s,a)+(\textbf{P}_h^{k_i}-P_h)[V_{h+1}^\star](s,a)\Big).$$
\end{lem}
\begin{proof}
    Take $Q_h^\star(s,a)$, and let $i$ be an index enumerating all the times $s^\bullet$ have been visited at step $h$, so that at episode $k$ $i$ can range up to $t:=n_h^k(s^\bullet)$. As $\sum_{i=0}^t \alpha_t^i=1$,
    $$Q_h^\star(s,a)=\sum_{i=0}^t \alpha_t^iQ_h^\star(s,a).$$
    The following identities follow from the definition of the operator
    \begin{align*}
        Q_h^\star(s,a)&=\alpha_t^0Q_h^\star(s,a)+\sum_{i=1}^t \alpha_t^iQ_h^\star(s,a)\\
        &=\alpha_t^0Q_h^\star(s,a)+\sum_{i=1}^t \alpha_t^i\left(r_h(s,a)+P_h[V_{h+1}^\star](s,a)\right)\\
        &=\alpha_t^0Q_h^\star(s,a)+\sum_{i=1}^t \alpha_t^i\left(r_h(s,a)+(P_h-\textbf{P}_h^{k_i})[V_{h+1}^\star](s,a)+\textbf{P}_h^{k_i}[V_{h+1}^\star](s,a)\right).
    \end{align*}
    
    Subtracting this to equation \eqref{eq:qlj} we get
    \begin{align*}
        Q_h^k(s,a)-Q_h^\star(s,a)&=\alpha_t^0(H-Q_h^\star(s,a))+\sum_{i=1}^t\alpha_t^i \Big(\textbf{P}_h^{k}[V_{h+1}^{k_i}-V_{h+1}^\star](s,a)+(\textbf{P}_h^{k_i}-P_h)[V_{h+1}^\star](s,a)+b_i\Big).
    \end{align*}
\end{proof}
The former result shows that, under a mild condition on the learning rate, the result $Q-$function estimated at any stage of every episode in a state-action pair $(s,a)$ may be written as a combination of the $Q-$functions evaluated in the same $(s,a)$ corresponding to all the episodes $k_i$ where the given exogenous state $s^\bullet$ was seen. The former result does not concern the variance of the estimate, and is valid pointwise. The next step is to show convergence to the optimal $Q^\star$-function via a concentration argument.
Again, due to the problem structure,
$$\E\left[(\textbf{P}_h^{k_i}-P_h)[V_{h+1}^\star](s,a)\Big|\Fs_{k_i}\right]=0,$$

as $s_{h+1}^{\bullet,k_i}$ is sampled from $p_h^\bullet(\cdot|s^\bullet)$. Moreover, the full term is bounded:

$$|(\textbf{P}_h^{k_i}-P_h)[V_{h+1}^\star](s,a)|\le H\qquad a.s.$$
Azuma-Hoeffding's inequality ensures that, w.p. $1-\delta$
\begin{equation}
    \sum_{i=1}^{t}\alpha_t^i(\textbf{P}_h^{k_i}-P_h)[V_{h+1}^\star](s,a)\le 2H\sqrt{\sum_{i=1}^t(\alpha_t^i)^2\log(1/\delta)}.
\end{equation}

If we choose $\alpha_t$ according to \ref{lem:alphachoice}, the former gives

$$2H\sqrt{\sum_{i=1}^t(\alpha_t^i)^2\log(1/\delta)}\le \sqrt{\frac{8H^3\log(1/\delta)}{t}}:=b_t.$$

Therefore, we have $Q_h^k(s,a)-Q_h^\star(s,a)$

\begin{align*}
    &= \alpha_t^0(H-Q_h^\star(s,a))+\sum_{i=1}^t\alpha_t^i \Big(\textbf{P}_h^{k_i}[V_{h+1}^{k_i}-V_{h+1}^\star](s,a)+(\textbf{P}_h^{k_i}-P_h)[V_{h+1}^\star](s,a)\Big)\\
    &\le \alpha_t^0(H-Q_h^\star(s,a))+\sum_{i=1}^t\alpha_t^i \textbf{P}_h^{k_i}[V_{h+1}^{k_i}-V_{h+1}^\star](s,a)+b_t.
\end{align*}

In this way, we have proved that the following inequality,
\begin{equation}
    Q_h^k(s,a)-Q_h^\star(s,a)\le \alpha_t^0H+\sum_{i=1}^t\alpha_t^i \textbf{P}_h^{k_i}[V_{h+1}^{k_i}-V_{h+1}^\star](s,a)+b_t.\label{eq:q0}
\end{equation}

Let us call
$\rho_h^k(s^\bullet):=\sup_{s^\diamond,a}|Q_h^k(s^\bullet,s^\diamond,a)-Q_h^\star(s^\bullet,s^\diamond,a)|$.
From the previous equation \ref{eq:q0}, we get

\begin{align*}
    \sum_{k=1}^K\rho_h^k(s^{k,\bullet}) &\le \sum_{k=1}^K \sup_{s^\diamond,a} |Q_h^k(s_h^{\bullet,k},s^\diamond,a)-Q_h^\star(s_h^{\bullet,k},s^\diamond,a)| \\
    &\le \sum_{k=1}^K \alpha_{n_h^k}^0H+b_{n_h^k}\\
    &\qquad +\sum_{k=1}^K \sum_{i=1}^{n_h^k}\alpha_{n_h^k}^i\sup_{s^\diamond,a} |\textbf{P}_h^{k_i}[V_{h+1}^{k_i}-V_{h+1}^\star](s_h^{\bullet,k},s^\diamond,a)| \\
    &= \sum_{k=1}^K \alpha_{n_h^k}^0H+b_{n_h^k}+\sum_{k=1}^K \sum_{i=1}^{n_h^k}\alpha_{n_h^k}^i \textbf{P}_h^{\bullet,k_i}[\rho_{h+1}^{k_i}](s^{k,\bullet}).
\end{align*}

Here, we denote $n_h^k:=n_h^k(s_h^{\bullet,k})$, the number of times $s_h^{\bullet,k}$ has appeared before the current episode.
If $\alpha_t^i$ is chosen according to \eqref{eq:lr}, we can rearrange the last part of the sum as follows and apply Lemma \ref{lem:alphachoice}:

$$\sum_{k=1}^K \sum_{i=1}^{n_h^k}\alpha_{n_h^k}^i \textbf{P}_h^{\bullet,k_i}[\rho_{h+1}^{k_i}](s^{k,\bullet})\le \sum_{k=1}^K \textbf{P}_h^{\bullet,k}[\rho_{h+1}^{k}](s^{k,\bullet})\underbrace{\sum_{t=n_h^k(s^{k,\bullet})+1}^\infty \alpha_t^{n_h^k(s^{k,\bullet})}}_{\le 1+1/H}.$$

But, by definition of $\textbf{P}_h^{\bullet,k}$,

\begin{equation}
    \sum_{k=1}^K \textbf{P}_h^{\bullet,k}[\rho_{h+1}^{k}](s^{k,\bullet})=\sum_{k=1}^K\rho_{h+1}^k(s^{k,\bullet})\label{eq:incrementalrho}.
\end{equation}

For the other term, we have by the pigeonhole principle
\begin{align}
    \sum_{k=1}^K b_{n_h^k} &=\sum_{s^\bullet\in \Ss^\bullet} \sum_{t=1}^{n_h^K}\sqrt{\frac{8H^3\log(1/\delta)}{t}}\\
    &\lesssim \sqrt{H^3|\Ss^\bullet|K\log(1/\delta)}.\label{eq:bonusbound}
\end{align}

By \eqref{eq:lr}, $\alpha_t=\frac{H+1}{H+t}$ which entails 
$$\alpha_t^0=\prod_{j=1}^t (1-\alpha_j)=0\times \prod_{j=2}^t (1-\alpha_j)=0.$$

therefore, 

\begin{equation}
    \sum_{k=1}^K \alpha_{n_h^k}^0H\le H|\Ss^\bullet|\label{eq:constantterm}.
\end{equation}

So by equations \eqref{eq:incrementalrho},\eqref{eq:bonusbound},\eqref{eq:constantterm}, the recursion is

$$\sum_{k=1}^K\rho_h^k(s^{k,\bullet})\le H|\Ss^\bullet|+\sqrt{H^3|\Ss^\bullet|K\log(1/\delta)}+(1+1/H)\sum_{k=1}^K\rho_{h+1}^k(s^{k,\bullet}).$$

As $\rho_{H}^k(\cdot)=0$, this entails

\begin{align*}
    \sum_{k=1}^K\rho_1^k(s^{k,\bullet})&\le \left(1+1/H\right)^H\left(H|\Ss^\bullet|+\sqrt{H^3|\Ss^\bullet|K\log(1/\delta)}\right)\\
    &\le e\left(H|\Ss^\bullet|+\sqrt{H^3|\Ss^\bullet|K\log(1/\delta)}\right)=:\text{Reg}(H,\Ss^\bullet,K).
\end{align*}

Finally, we are able to bound the full regret. 

\regretboundtwo*
\begin{proof}
    For all $h\le H$ fix
    
    $$R_{h,K}:=\sum_{k=1}^K V_h^{\star}(s_h^k)-V_h^{\pi_k}(s_h^k)\qquad R_K=R_{1,K}\qquad R_{H,K}=0.$$
    
    we have for $h\le H-1$
    
    \begin{align*}
        \sum_{k=1}^K V_h^{\star}(s_h^k)-V_h^{\pi_k}(s_h^k)&=\sum_{k=1}^K Q_h^{\star}(s_h^k, \pi_h^\star(s_h^k))-Q_h^{\pi_k}(s_h^k,a_h^k)\\
        &= \sum_{k=1}^K Q_1^{\star}(s_h^k, \pi_h^\star(s_h^k))-Q_h^{\star}(s_h^k, a_h^k)+Q_h^{\star}(s_h^k,a_h^k)-Q_h^{\pi_k}(s_h^k,a_h^k)\\
        &= \sum_{k=1}^K Q_1^{\star}(s_h^k, \pi_h^\star(s_h^k))-Q_h^{\star}(s_h^k, a_h^k)+\sum_{k=1}^K P_hV_{h+1}^{\star}(s_h^k, a_h^k)-P_hV_{h+1}^{\pi_k}(s_h^k, a_h^k)\\
        &=\sum_{k=1}^K Q_h^{\star}(s_h^k, \pi_h^\star(s_h^k))-Q_h^{\star}(s_h^k, a_h^k)+\sum_{k=1}^K V_{h+1}^{\star}(s_{h+1}^k)-V_{h+1}^{\pi_k}(s_{h+1}^k)\\
        &\qquad +\sum_{k=1}^K \underbrace{\left(P_hV_{h+1}^{\star}(s_h^k, a_h^k)-P_hV_{h+1}^{\pi_k}(s_h^k, a_h^k)\right)-\left(V_{h+1}^{\star}(s_{h+1}^k)-V_{h+1}^{\pi_k}(s_{h+1}^k)\right)}_{\eta_h^k}\\
        &\le \underbrace{\sum_{k=1}^K Q_h^{k}(s_h^k, \pi_h^\star(s_h^k))-Q_h^{k}(s_h^k, a_h^k)}_{\le 0}+2\sum_{k=1}^K\rho_h^k(s^{k,\bullet})+\sum_{k=1}^K\eta_h^k+R_{h+1,K}\\
        &\lesssim 2\text{Reg}(H,\Ss^\bullet,K)+R_{h+1,K}.
    \end{align*}
    
    Here, the last step comes from the fact that $\sum_{k=1}^K\eta_h^k$ is martingale difference sequence with increments bounded by $2H$, so its sum is bounded by $2H\sqrt{K\log(1/\delta)}$ with probability at least $1-\delta$ (Azuma-Hoeffding), and the former is bounded by $\text{Reg}(H,\Ss^\bullet,K)$.
    
    Completing the recursion, we get
    $R_{h,K}\le 2\text{Reg}(H,\Ss^\bullet,K)+R_{h+1,K}$, which entails
    
    $$R_K\le H\text{Reg}(H,\Ss^\bullet,K)\le 2e\left(H^2|\Ss^\bullet|+\sqrt{H^5|\Ss^\bullet|K\log(1/\delta)}\right).$$
    
    This ends the proof.
\end{proof}

\subsection{Proof of the Lower Bound}\label{app:lowerbound}
In this section, we prove the lower bound.

\begin{figure}[t]
    \centering
\tikzset{every picture/.style={line width=0.75pt}} 

\begin{tikzpicture}[x=0.75pt,y=0.75pt,yscale=-1,xscale=1]

\draw [fill=black!10]  (295,977.9) .. controls (295,965.92) and (305.75,956.2) .. (319,956.2) .. controls (332.25,956.2) and (343,965.92) .. (343,977.9) .. controls (343,989.88) and (332.25,999.6) .. (319,999.6) .. controls (305.75,999.6) and (295,989.88) .. (295,977.9) -- cycle ;
\draw [fill=black!10]  (421,843.7) .. controls (421,831.72) and (431.75,822) .. (445,822) .. controls (458.25,822) and (469,831.72) .. (469,843.7) .. controls (469,855.68) and (458.25,865.4) .. (445,865.4) .. controls (431.75,865.4) and (421,855.68) .. (421,843.7) -- cycle ;
\draw [fill=black!10]  (154,839.7) .. controls (154,827.72) and (164.75,818) .. (178,818) .. controls (191.25,818) and (202,827.72) .. (202,839.7) .. controls (202,851.68) and (191.25,861.4) .. (178,861.4) .. controls (164.75,861.4) and (154,851.68) .. (154,839.7) -- cycle ;
\draw [fill=black!10]  (107,739.7) .. controls (107,727.72) and (117.75,718) .. (131,718) .. controls (144.25,718) and (155,727.72) .. (155,739.7) .. controls (155,751.68) and (144.25,761.4) .. (131,761.4) .. controls (117.75,761.4) and (107,751.68) .. (107,739.7) -- cycle ;
\draw [fill=black!10]  (188,741.7) .. controls (188,729.72) and (198.75,720) .. (212,720) .. controls (225.25,720) and (236,729.72) .. (236,741.7) .. controls (236,753.68) and (225.25,763.4) .. (212,763.4) .. controls (198.75,763.4) and (188,753.68) .. (188,741.7) -- cycle ;
\draw    (319,956.2) .. controls (319,924.52) and (196.49,911.46) .. (178.51,862.88) ;
\draw [shift={(178,861.4)}, rotate = 72.19] [color={rgb, 255:red, 0; green, 0; blue, 0 }  ][line width=0.75]    (10.93,-3.29) .. controls (6.95,-1.4) and (3.31,-0.3) .. (0,0) .. controls (3.31,0.3) and (6.95,1.4) .. (10.93,3.29)   ;
\draw    (319,956.2) .. controls (318.02,924.68) and (442.19,901.89) .. (444.96,867) ;
\draw [shift={(445,865.4)}, rotate = 88.4] [color={rgb, 255:red, 0; green, 0; blue, 0 }  ][line width=0.75]    (10.93,-3.29) .. controls (6.95,-1.4) and (3.31,-0.3) .. (0,0) .. controls (3.31,0.3) and (6.95,1.4) .. (10.93,3.29)   ;
\draw    (178,818) .. controls (178,784.54) and (214.26,809.69) .. (212.08,764.78) ;
\draw [shift={(212,763.4)}, rotate = 86.33] [color={rgb, 255:red, 0; green, 0; blue, 0 }  ][line width=0.75]    (10.93,-3.29) .. controls (6.95,-1.4) and (3.31,-0.3) .. (0,0) .. controls (3.31,0.3) and (6.95,1.4) .. (10.93,3.29)   ;
\draw    (178,818) .. controls (178,792.72) and (132.86,794.13) .. (131.06,763.33) ;
\draw [shift={(131,761.4)}, rotate = 90] [color={rgb, 255:red, 0; green, 0; blue, 0 }  ][line width=0.75]    (10.93,-3.29) .. controls (6.95,-1.4) and (3.31,-0.3) .. (0,0) .. controls (3.31,0.3) and (6.95,1.4) .. (10.93,3.29)   ;
\draw    (445,822) .. controls (445,788.54) and (481.26,813.69) .. (479.08,768.78) ;
\draw [shift={(479,767.4)}, rotate = 86.33] [color={rgb, 255:red, 0; green, 0; blue, 0 }  ][line width=0.75]    (10.93,-3.29) .. controls (6.95,-1.4) and (3.31,-0.3) .. (0,0) .. controls (3.31,0.3) and (6.95,1.4) .. (10.93,3.29)   ;
\draw    (445,822) .. controls (445,796.72) and (399.86,798.13) .. (398.06,767.33) ;
\draw [shift={(398,765.4)}, rotate = 90] [color={rgb, 255:red, 0; green, 0; blue, 0 }  ][line width=0.75]    (10.93,-3.29) .. controls (6.95,-1.4) and (3.31,-0.3) .. (0,0) .. controls (3.31,0.3) and (6.95,1.4) .. (10.93,3.29)   ;

\draw (310,968.6) node [anchor=north west][inner sep=0.75pt]   [align=left] {$\displaystyle z_{0}^\bullet$};
\draw (437,835.4) node [anchor=north west][inner sep=0.75pt]   [align=left] {$\displaystyle z_{N}^\bullet$};
\draw (169,830.4) node [anchor=north west][inner sep=0.75pt]   [align=left] {$\displaystyle z_{1}^\bullet$};
\draw (303,831.6) node [anchor=north west][inner sep=0.75pt]   [align=left] {...};
\draw (119,731.4) node [anchor=north west][inner sep=0.75pt]   [align=left] {$\displaystyle z_{1,1}^{\bullet }$};
\draw (200,731.4) node [anchor=north west][inner sep=0.75pt]   [align=left] {$\displaystyle z_{1,2}^{\bullet }$};
\draw (436,748.6) node [anchor=north west][inner sep=0.75pt]   [align=left] {...};
\draw (207,920.6) node [anchor=north west][inner sep=0.75pt]   [align=left] {$\displaystyle 1/N$};
\draw (391,922.6) node [anchor=north west][inner sep=0.75pt]   [align=left] {$\displaystyle 1/N$};
\draw (130,790.6) node [anchor=north west][inner sep=0.75pt]   [align=left] {$\displaystyle p_1$};
\draw (205,790.6) node [anchor=north west][inner sep=0.75pt]   [align=left] {$\displaystyle 1-p_1$};
\end{tikzpicture}
    \caption{The uncontroallable part of the hard MDP family}    \label{fig:lower}
\end{figure}
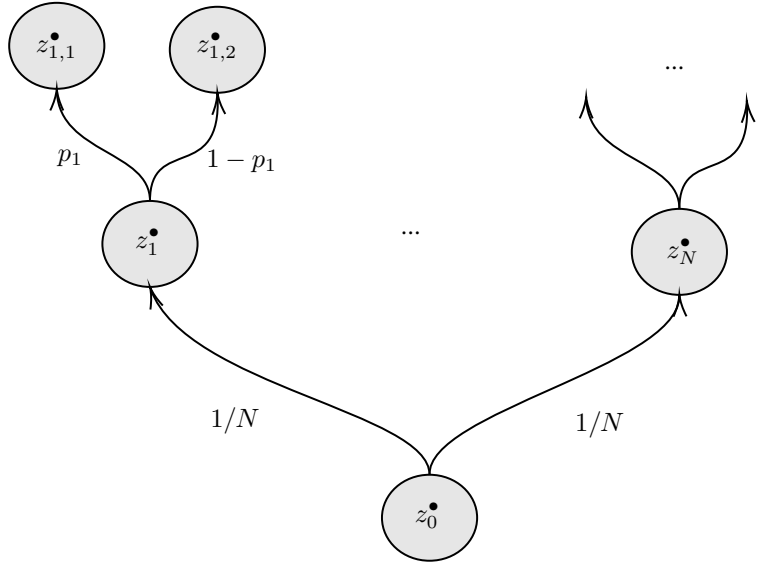

\lowerbound*
\begin{proof}
    Fix $N\in \mathbb N$ and consider the following family of PC-MDP instances:
    \begin{itemize}
        \item The uncontrollable state space $\Ss^\bullet$ has $3N+1$ states: $z_0^\bullet$, the original state, which brings to $z_1^\bullet,\dots z_N^\bullet$, each with probability $1/N$, which in turn bring to two states (denoted as $z_{i,1}^\bullet$ and $z_{i,2}^\bullet$) with probability $p_i$ and $1-p_i$ respectively.

        \item The action space is $\As=\{-1,+1\}$.

        \item The space is $\Ss^\diamond=\{0,-1,+1\}$. In the first two steps of the MDP, the agent always has $z^\diamond=0$. At the third step, the agent chooses between ending in $\{-1,+1\}$ according to the action at step $1$.

        \item The reward is identically zero except for the last step. There,
        $$r_3(z^\bullet,z^\diamond,a)=\begin{cases}
        z^\diamond \qquad &\text {if we are in a state }z_{i,1}^\bullet\\
        -z^\diamond \qquad &\text {if we are in a state }z_{i,2}^\bullet
        \end{cases}$$
    \end{itemize}

    The instance is completely identified by the vector $\bm p=[p_i]_{i=1}^n$. In figure \ref{fig:lower}, we provide a visual representation of the part of these MDPs relative to state $\Ss^\bullet$.

    By design of the problem, each instance reduces to learning with the different values of $z^\bullet_i$, for $i=1,\dots n$ being completely disentangled. The total regret on any algorithm thus reduces to

    \begin{align*}
        R_K&=\sum_{k=1}^K \left(V_1^\star(s_1^k) - V_1^{\pi_k}(s_1^k) \right)\\
        &=\sum_{i=1}^N\sum_{k=1}^K \left(V_1^\star(s_1^k) - V_1^{\pi_k}(s_1^k) \right)\bm 1(s_2^{k,\bullet}=i)\\
        &=\sum_{i=1}^N\sum_{k=1}^K \left(V_1^\star(z_{i,2}^\bullet) - V_1^{\pi_k}(z_{i,2}^\bullet) \right)\bm 1(s_2^{k,\bullet}=z_{i,2}^\bullet)\\
        &=:\sum_{i=1}^N R_{K_i,i}.
    \end{align*}

    Where $K_i$ is the number of episodes such that $s_2^{k,\bullet}=i$ and $R_{K_i,i}$ corresponds to the regret incurred while learning from state $z_i^\bullet$ for $K_i$ episodes. This decomposition holds since the reward and state $s^\diamond$ are null before step $2$ and, for any $i$, $p_i$, which is the only thing that the agent can learn, only influences the state-value function of $z_i^\bullet$.

    Now, let us fix $i$ and consider the problem of learning from state $z_i^\bullet$. The reward is determined in the following way:

    \begin{itemize}
        \item If the agent chooses action $a=1$, then $z^\diamond=1$, and the reward is $+1$ w.p. $p_i$ and $-1$ w.p. $1-p_i$, so a Dirchlet random variable with $2p_i-1$ as expected value
        \item Vice-versa, if $a=-1$,the reward corresponds to the  previous random variable times $-1$, that is, a Dirchlet of parameter $1-2p_i$. 
    \end{itemize}

    Let us call $X_{t}$, for $t=1,\dots K_i$ the random variable $\bm 1(s_2^{t,\bullet})$ (here $t$ only counts the episodes where the agent visits $z_i^\bullet$). For what we said, the agent's reward after action $a$ is $aX_t$, and the regret writes as

    \begin{equation}
    R_{K_i,i}=\underbrace{\max\{4p_i-2,0\}T_1}_{\text{when action }1\text{ is optimal}} + \underbrace{\max\{2-4p_i,0\}(K_i-T_1)}_{\text{when action }2\text{ is optimal}}\label{eq:neoregdef}
    \end{equation}

    Consider now two instances: (I) when $p_i=1/2+\Delta$ and (II) $p_i=1/2-\Delta$, for $\Delta=\frac{1}{4}K_i^{-1/2}$. 
    These two events introduces two probabilities $\Prob^{(I)}$ and $\Prob^{(II)}$ in the sequence of random variables $X_1,\dots X_{K_i}$. By independence, and the fact that $X_1,\dots X_{K_i}$ have Rademacher distribution, the KL divergence of the two probabilities is the following:
    $$\text{D}(\Prob^{(I)}|\Prob^{(II)})=2K_i\Delta\log\left(\frac{1/2+\Delta}{1/2-\Delta}\right)\approx 8K_i\Delta^2=1/2.$$    
    Now, for every learning algorithm, the random variable $T_1$ is measurable w.r.t. the sequence $X_1,\dots X_{K_i}$.  
    Therefore, the Bretagnolle-Huber inequality ensures

    \begin{equation}
        \Prob^{(I)}(T_1\le K_i/2)+\Prob^{(II)}(T_1> K_i/2)\ge \exp(-\text{D}(\Prob^{(I)}|\Prob^{(II)}))/2\approx e^{-1/2}/2\label{eq:uber}
    \end{equation}

    Therefore, we can use \eqref{eq:neoregdef} to compute the regret in the two cases

    \begin{enumerate}
        \item In instance $(I)$ the regret satisfies
        $$\E^{(I)}[R_{K_i,i}]\ge K_i/2\times K_i^{-1/2}\Prob^{(I)}(T_1\le K_i/2)$$
        \item In instance $(II)$ the regret satisfies
        $$\E^{(II)}[R_{K_i,i}]\ge K_i/2\times K_i^{-1/2}\Prob^{(II)}(T_1> K_i/2)$$
    \end{enumerate}
    Putting the two equations together, thanks to \eqref{eq:uber}, we have for any algorithm

    $$\E^{(I)}[R_{K_i,i}]+\E^{(II)}[R_{K_i,i}]\ge \frac{e^{-1/2}}{4} \sqrt{K_i},$$

    so that in particular $\E^{(I)}[R_{K_i,i}]\ge \frac{e^{-1/2}}{8} \sqrt{K_i}$ in at least one of the two instances.
    Coming back to our full regret bound, we note that $K_i$ follows a binomial distribution $\text{Bin}(K,N^{-1})$.
    Thus, 
    \begin{align*}
        \E\left[\sqrt K_i\right] &\ge \sqrt{\frac{K}{N}}\Prob(K_i\ge K/N)=\sqrt{\frac{K}{4N}}. 
    \end{align*}
    Therefore, taking the worst instances for every $i$,

    \begin{align*}
        \E[R_K]&=\E\left[\sum_{i=1}^N R_{K_i,i}\right]\ge \frac{e^{-1/2}}{8}\sum_{i=1}^N \E\left[\sqrt{K_i}\right]\\
        & = \frac{e^{-1/2}}{8}N\sqrt{\frac{K}{4N}}\ge \frac{1}{30}\sqrt{NK}\ge \frac{1}{90}\sqrt{S^\bullet K}.
    \end{align*}
\end{proof}
\section{Complexity Analysis}\label{app:complexity}
We analyze the space and time complexity of our approaches compared to standard baseline algorithms in the tabular setting. We consider both model-based methods (\textsc{UCBVI} vs.\ \textsc{ExAVI}) and model-free methods (\textsc{QL} vs.\ \textsc{ExAQ}).
Let $S^\diamond = |\mathcal{S}^\diamond|$ and $S^\bullet = |\mathcal{S}^\bullet|$ denote the cardinalities of the controllable and exogenous state spaces, respectively. The size of the global state space is $S = S^\diamond \cdot S^\bullet$, and $A = |\mathcal{A}|$ denotes the size of the action space.

\paragraph{Space Complexity}~~Standard \textsc{UCBVI} requires storing a value table ($Q_h$) for each step $h$ of the horizon $H$, with size $\Os(HSA)$. Crucially, it must also estimate the full transition matrix $\hat{P}(s' \mid s, a)$, which requires storing counts for every state-action-state tuple for each step of the episode. In the worst case (dense dynamics), this incurs a space complexity of $\Os(H S^2 A)$.
In contrast, \textsc{ExAVI} exploits the PCMDP structure by decomposing the transition dynamics. It assumes knowledge of the controllable kernel $p^\diamond$ (size $\Os(HS^{\diamond 2} A)$) and learns only the exogenous transition matrix $\hat{p}^\bullet(s^{\bullet_{h+1}} \mid s^\bullet_h)$, which has a size of $\Os(HS^{\bullet 2})$. If $S^\bullet \ll S$ and the components are decoupled, \textsc{ExAVI} achieves a significant reduction in model storage requirements: from multiplicative $(S^\diamond S^\bullet)^2$ to additive $S^{\diamond 2} + S^{\bullet 2}$.

Both standard \textsc{QL} and \textsc{ExAQ} store a Q-table of size $\Os(HSA)$. However, \textsc{ExAQ} additionally stores the controllable transition kernel $p^\diamond$, requiring $\Os(HS^{\diamond 2} A)$ space. In many practical domains (e.g., grid worlds, kinematics), these dynamics are localized or can be computed on-the-fly (as in our experiments), reducing the effective storage requirement. Consequently, when $S^\bullet \gg S^\diamond$ (as in our trading task), the memory overhead of \exaq is negligible compared to the size of the Q-table itself.

\paragraph{Time Complexity}~~For model-based methods, the computational bottleneck is the planning phase (backward induction). \textsc{UCBVI} performs a summation over the full state space $S$ for every state-action pair, scaling with $\Os(S^2 A)$. \textsc{ExAVI} mitigates this by computing expectations over $S^\diamond$ and $S^\bullet$ separately, reducing the complexity to $\Os((S^\diamond + S^\bullet) S A)$ in the best case.

For model-free methods, standard \textsc{QL} performs a single scalar update per step, $\Os(1)$. In contrast, \textsc{ExAQ} executes a \emph{counterfactual update} of the Q-table's controllable region. Upon observing an exogenous transition $s^\bullet_h \to s^\bullet_{h+1}$, the agent updates the Q-values for the current exogenous context across all reachable controllable states and actions. This results in a per-step complexity of $\Os(S^\diamond A)$.

\begin{table}[h]
    \centering
    \caption{Summary table with complexity analysis.}
    \label{tab:complexity}
    \vspace{0.2cm}
    \setlength{\tabcolsep}{5pt}
    \begin{tabular}{llcc}
        \toprule
        \textbf{Type} & \textbf{Algorithm} & \textbf{Space (Model + Value)} & \textbf{Time (Per Step / Plan)} \\
        \midrule
        \multirow{2}{*}{Model-Based} 
        & \textsc{UCBVI} & $\Os(HS^2 A + HSA)$ & $\Os(S^2 A)$ (Planning) \\
        & \textsc{ExAVI} & $\Os(HS^{\bullet 2} + HS^{\diamond 2}A + HSA)$ & $\Os((S^\diamond \!+\! S^\bullet)SA)$ (Planning) \\
        \midrule
        \multirow{2}{*}{Model-Free} 
        & \textsc{QL} & $\Os(HSA)$ & $\Os(1)$ \\
        & \textsc{ExAQ} & $\Os(HSA + HS^{\diamond 2}A)$ & $\Os(S^\diamond A)$ \\
        \bottomrule
    \end{tabular}
\end{table}

\section{Additional Results and Settings}\label{app:additional-results}
In this part of the appendix, we provide additional results and details about previous presented experiment. Moreover, we introduce the \texttt{ElevatorEnv} task on which we evaluate both \emph{model-based} and \emph{model-free} approaches. 

\subsection{Additional Experiment} \label{app:elevator-details}
In the following, we present the remaining experiment, excluded from the main paper for the sake of space. We start by providing the description of the task.

\paragraph{Elevator Dispatching Environment.}~~This task is a simplified adaptation of the elevator scheduling problem introduced by~\citet{crites1995elevator} and provided within the Gym4ReaL library~\citep{salaorni2025gym4real}. It simulates a \emph{peak-down traffic} scenario, typical of office buildings at the end of a workday. A single elevator serves a building with $F$ floors, transporting employees to the ground floor ($f=0$).
New passengers arrive at each floor $f \in \{1, \dots, F\}$ according to an independent \emph{Poisson process} with rate $\lambda_f$. Arriving passengers join a queue at their respective floor, provided the queue length is below a threshold $W_{\max}$; otherwise, they act as lost demand (e.g., taking the stairs).

The objective is to minimize the total cumulative waiting time of all passengers. Formally, we define the \emph{controllable state} at step $h$ as $s^\diamond_h = (\nu_h, \psi_h, \mathbf{w}_h)$, where $\nu_h \in \{0, \dots, F\}$ denotes the elevator's current floor, $\psi_h$ indicates the number of passengers currently inside the elevator (up to capacity $\Psi_{\max}$), and $\mathbf{w}_h \in \mathbb{N}^{F}$ represents the vector of queue lengths at each floor.
The \emph{exogenous state} is defined as the vector of new arrivals $s^\bullet_h = \boldsymbol{\kappa}_h \in \mathbb{N}^{F}$, where each component $\kappa_{f,h} \sim \text{Poisson}(\lambda_f)$ represents the number of users arriving at floor $f$ at step $h$. Note that while the arrival process $\boldsymbol{\kappa}_h$ is exogenous, the evolution of the queue lengths $\mathbf{w}_h$ is endogenous, as it depends on both the stochastic arrivals and the agent's pick-up actions.

The action space is discrete, $a_h \in \{\text{\textit{up}}, \text{\textit{down}}, \text{\textit{open}}\}$, allowing the elevator to move between floors or open doors to board/alight passengers.
The reward function is designed to penalize delays while incentivizing successful transport:
\begin{equation}
    r_h = -\left(\sum_{f=1}^{F} w_{f,h} + \psi_h\right) + \mathbbm{1}_{(\psi_h=0)} \cdot \beta \, \psi_{h-1},
\end{equation}
where the first term penalizes the total waiting time (users in queues plus users in the elevator), and the second term grants a bonus $\beta > 0$ for every passenger successfully offloaded at the ground floor.

\paragraph{Results on \texttt{ElevatorEnv}.}~~Figure~\ref{fig:elev-comparison} illustrates the learning performance on the \texttt{ElevatorEnv}. In the model-based setting (Figure~\ref{fig:elev-model-based}), \exavi demonstrates immediate convergence to the optimal policy, maintaining a steady return of approximately $350$. In contrast, \textsc{UCBVI} converges to a significantly suboptimal plateau ($\approx 250$). This suggests that the standard exploration bonuses were insufficient to guide the baseline out of local optima, whereas \exavi's structured updates allowed it to identify the optimal dispatching strategy instantly.
A similar trend is observed in the model-free comparison (Figure~\ref{fig:elev-model-free}). \exaq solves the task almost instantaneously, leveraging the known controllable dynamics to propagate value information across unvisited states. Conversely, standard \textsc{QL} exhibits a slow, asymptotic climb, requiring nearly $3,000$ episodes to match the performance that \textsc{ExAQ} achieved in the first few trials. This massive gap in sample efficiency highlights the effectiveness of counterfactual updates, made possible by isolating the exogenous state components.

\begin{figure}[ht]
    \centering
    \begin{subfigure}[b]{0.48\textwidth}
        \centering
        \includegraphics[width=\textwidth]{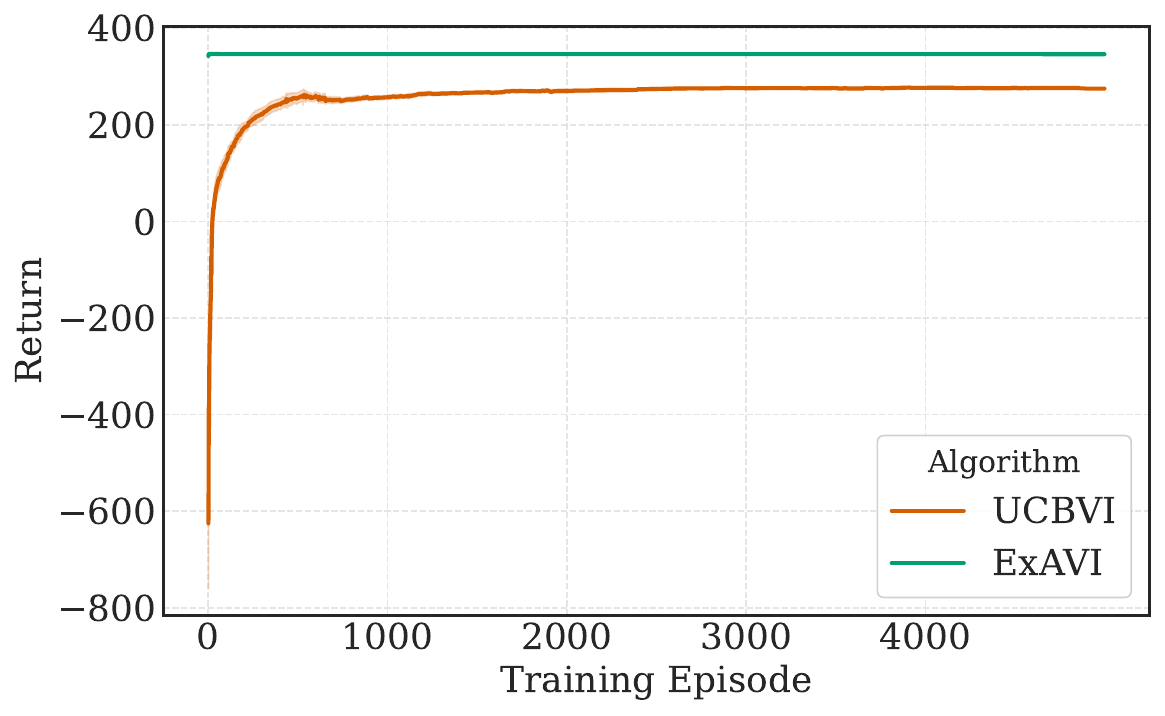}
        \caption{Model-based (\exavi vs.\ \textsc{UCBVI})}
        \label{fig:elev-model-based}
    \end{subfigure}
    \hfill
    \begin{subfigure}[b]{0.48\textwidth}
        \centering
        \includegraphics[width=\textwidth]{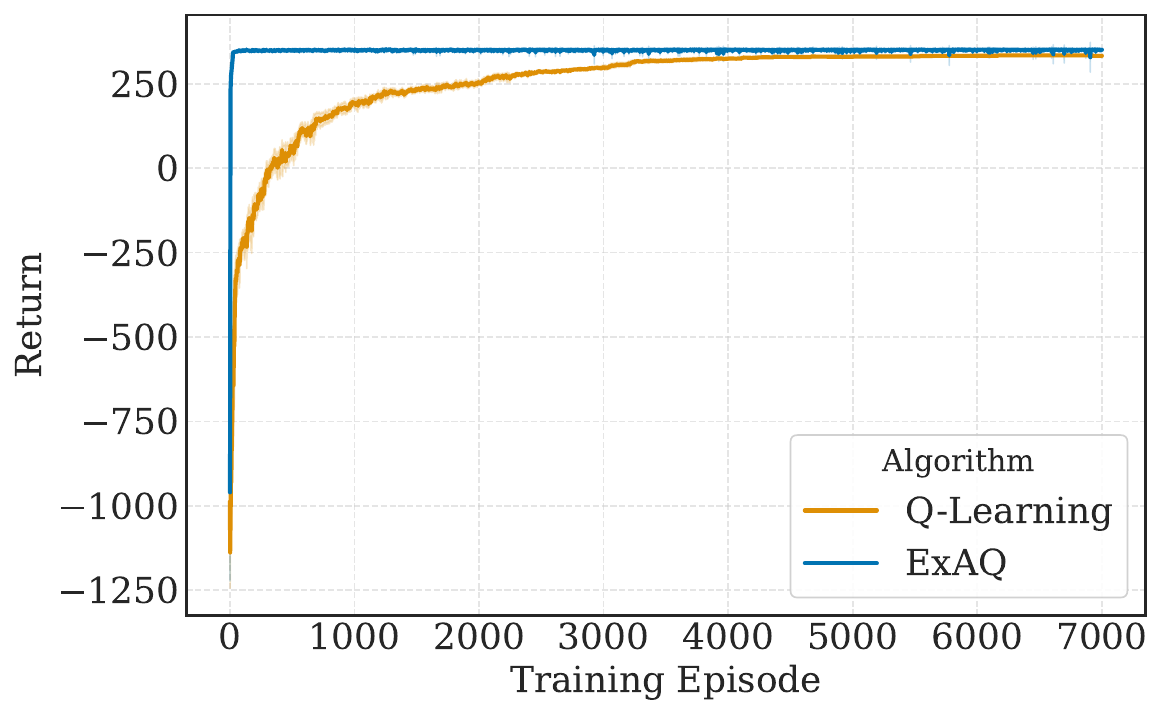}
        \caption{Model-free (\exaq vs.\ \textsc{QL})}
        \label{fig:elev-model-free}
    \end{subfigure}
    \caption{Comparative learning curves for the tiny version of \texttt{ElevatorEnv}, averaged over $10$ random seeds. Shaded regions denote $95\%$ confidence intervals. Figure (a) Model-based performance. Figure (b) Model-free performance.}
    \label{fig:elev-comparison}
\end{figure}

\paragraph{Qualitative Analysis on \texttt{TradingEnv}.}~~Complementing the main quantitative results, we qualitatively analyze the execution profiles learned by the different agents to gain intuition into their strategies. Figure~\ref{fig:trading-strategies} depicts the average inventory depletion over time for the optimal execution task. We compare our approach against standard RL baselines (\textsc{PPO}, \textsc{QL}) and the industry-standard \textsc{TWAP} (Time-Weighted Average Price) strategy.

From the plot, distinct behavioral patterns emerge. The \textsc{TWAP} strategy follows a deterministic linear liquidation path, splitting the order evenly across the entire horizon regardless of market conditions. In contrast, the standard model-free baselines (\textsc{QL} and \textsc{PPO}) exhibit extremely aggressive behavior, dumping nearly the entire inventory within the first $25$ steps. This suggests a failure to effectively balance risk and execution cost; these agents likely converged to a ``panic selling'' policy to avoid the variance of future price movements.
Conversely, \exaq demonstrates a sophisticated, adaptive strategy. It liquidates faster than \textsc{TWAP} to mitigate exposure to price volatility but maintains a much smoother profile than the other RL agents, fully clearing the position around step 90. The convex shape of the \exaq curve resembles theoretical optimal execution trajectories (e.g., Almgren-Chriss solutions), indicating that the agent successfully learned to balance the urgency of liquidation with the opportunity cost of aggressive selling.

\begin{figure}[t]
    \centering
    \includegraphics[width=0.6\linewidth]{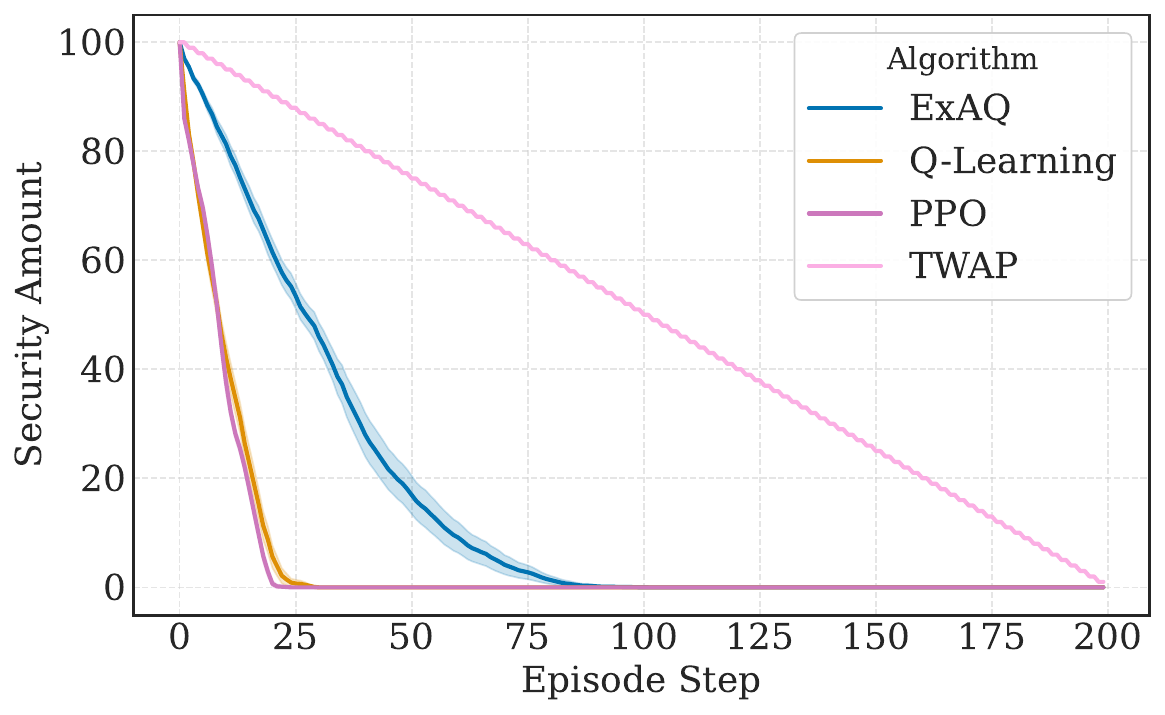}
    \caption{Inventory liquidation profiles for the optimal execution problem. \textsc{ExAQ} (blue) discovers a balanced strategy between the passive \textsc{TWAP} (pink) and the aggressive dumping of \textsc{PPO}/\textsc{QL}.}
    \label{fig:trading-strategies}
\end{figure}

\subsection{Environment Parametrization}\label{app:env-details}
In this section, we detail the configuration parameters to ensure the reproducibility of the experiments. Figure~\ref{fig:taxi-screenshot} illustrates the modified taxi environment, highlighting the three central cells (purple) where stochastic traffic congestion occurs, alongside the specific hyperparameters used for the simulation (in Table~\ref{tab:taxi-params}). We report the configuration of the trading and the elevator environments in Table~\ref{tab:trading-params} and Table~\ref{tab:elevator-params}, respectively.

\begin{figure}[ht]
    \centering
    \begin{minipage}[t]{0.45\textwidth}
        \centering
        \vspace{0pt} 
        \includegraphics[width=0.8\linewidth]{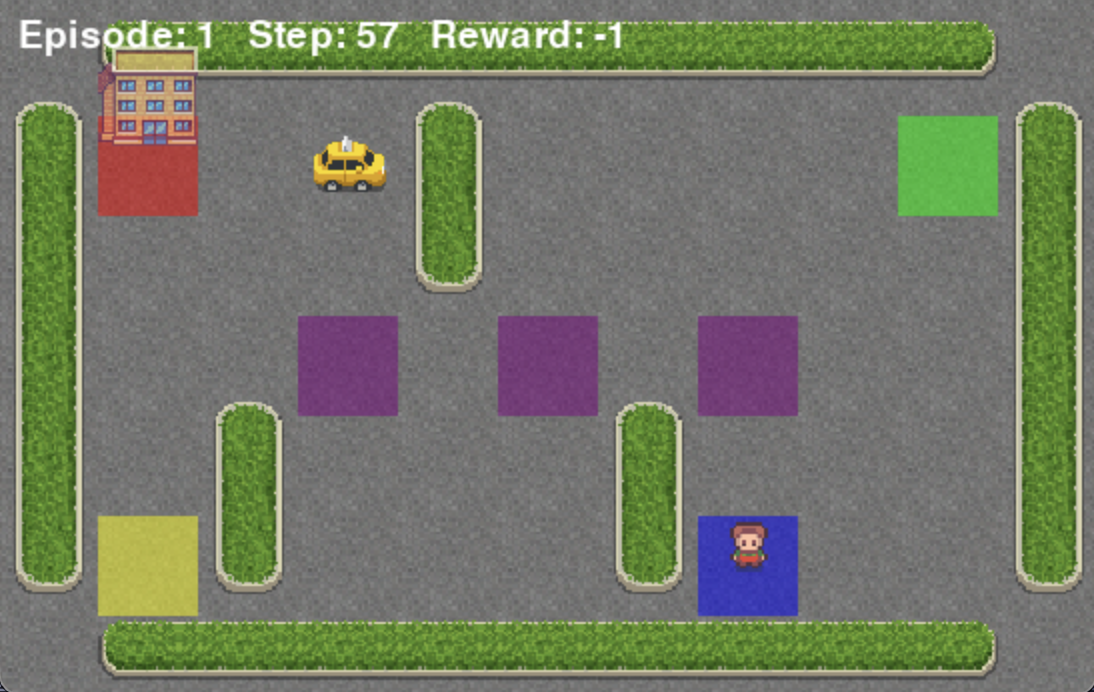}
        \captionof{figure}{Visual rendering of the environment.} 
        \label{fig:taxi-screenshot}
    \end{minipage}
    \hfill
    \begin{minipage}[t]{0.50\textwidth}
        \centering
        \vspace{0pt} 
        \captionof{table}{\texttt{TaxiEnv} parameters.}
        \label{tab:taxi-params} 
        \begin{tabular}{lr}
            \toprule
            \textbf{Parameter} & \textbf{Value} \\
            \midrule
            Horizon ($H$) & $200$ \\
            Grid Size & $5 \times 5$ \\
            Traffic Locations & $(2,1), (2,2), (2,3)$ \\
            Traffic Prob. ($p_{\text{tr}}$) & $0.3$ \\
            Passenger Locations & 4 Corners + Taxi \\
            \bottomrule
        \end{tabular}
    \end{minipage}
\end{figure}

\begin{table}[ht]
    \centering
    \begin{minipage}[t]{0.48\textwidth}
        \centering
        \caption{\texttt{TradingEnv} parameters.}
        \label{tab:trading-params}
        \begin{tabular}{lr}
            \toprule
            \textbf{Parameter} & \textbf{Value} \\
            \midrule
            \multicolumn{2}{l}{\textit{General Specs}} \\
            \midrule
            Horizon ($H$) & $200$ \\
            Price Range & $[90, 110]$ \\
            Initial Price ($\omega_0$) & $100.0$ \\
            Volatility ($\sigma$) & $0.3$ \\
            Drift ($\mu$) & $0.0$ \\
            Price Granularity & $0.02$ \\
            \midrule
            \multicolumn{2}{l}{\textit{Execution \& Impact}} \\
            \midrule
            Initial Inventory ($u_0$) & $100$ \\
            Risk Aversion ($\lambda$) & $100.0$ \\
            Transaction Cost ($\epsilon$) & $0.0625$ \\
            Adjusted Temp. Impact ($\tilde \eta$) & $2 \times 10^{-5}$ \\
            \bottomrule
        \end{tabular}
    \end{minipage}
    \hfill 
    \begin{minipage}[t]{0.48\textwidth}
        \centering
        \caption{\texttt{ElevatorEnv} parameters.}
        \label{tab:elevator-params}
        \begin{tabular}{lr}
            \toprule
            \textbf{Parameter} & \textbf{Value} \\
            \midrule
            \multicolumn{2}{l}{\textit{General Specs}} \\
            \midrule
            Horizon ($H$) & $300$ \\
            Floors ($F$) & $3$ (Indices $0,1,2$) \\
            Capacity ($\Psi_{\max}$) & $2$ \\
            Goal Floor & $0$ (Ground) \\
            \midrule
            \multicolumn{2}{l}{\textit{Passenger Dynamics}} \\
            \midrule
            Arrival Rate ($\lambda$) & $[0.01, 0.2]$ \\
            Max Queue ($W_{\max}$) & $2$ \\
            Max Batch ($\kappa_{\max}$) & $2$ \\
            \midrule
            \multicolumn{2}{l}{\textit{Rewards}} \\
            \midrule
            Delivery Bonus ($\beta$) & $+10.0$ \\
            Waiting Penalty & $-1.0$ \\
            \bottomrule
        \end{tabular}
    \end{minipage}
\end{table}

\section{Reproducibility}
The experiments have been conducted on a MacBook Pro 2023 equipped with an Apple M2 Pro CPU and 16 GB of RAM. The repository containing the source code will be disclosed in the camera-ready version.

\subsection{Hyperparameters}
Hereafter, to foster reproducibility, we report the hyperparameters utilized to obtain the presented results for every specific environment. 

\begin{table}[h]
    \centering
    \caption{Hyperparameter of model-based methods across \texttt{TaxiEnv} and \texttt{ElevatorEnv} tasks. Shared parameters are aligned, while specific tuning differences are highlighted.}
    \label{tab:hyperparameters_comparison}
    \vspace{0.2cm}
    \setlength{\tabcolsep}{3.5pt}
    \begin{tabular}{lcccc}
        \toprule
         & \multicolumn{2}{c}{\textbf{Taxi}} & \multicolumn{2}{c}{\textbf{Elevator}}\\
        \cmidrule(lr){2-3} \cmidrule(lr){4-5}
        \textbf{Parameter} & \textsc{UCBVI} & \textsc{ExAVI} & \textsc{UCBVI} & \textsc{ExAVI} \\
        \midrule
        \multicolumn{5}{l}{\textit{General Settings}} \\
        \midrule
        Discount Factor ($\gamma$) & $1.0$ & $1.0$ & $1.0$ & $1.0$ \\
        Training Seeds & \multicolumn{2}{c}{$[1 \dots 10]$} & \multicolumn{2}{c}{$[1 \dots 10]$} \\
        \midrule
        \multicolumn{5}{l}{\textit{Optimization}} \\
        \midrule
        Bonus ($C$) & $0.5$ & $-$ & $0.5$ & $-$ \\
        Delta ($\delta$) & $10^{-6}$ & $-$ & $10^{-6}$ & $-$ \\
        \midrule
        \multicolumn{5}{l}{\textit{Execution}} \\
        \midrule
        Total Episodes & \multicolumn{2}{c}{$5,000$} & \multicolumn{2}{c}{$5,000$} \\
        Eval Length & \multicolumn{2}{c}{50 eps} & \multicolumn{2}{c}{50 eps}\\
        \bottomrule
    \end{tabular}
\end{table} 

\begin{table}[h]
    \centering
    \caption{Hyperparameter of model-free methods across \texttt{TaxiEnv},  \texttt{ElevatorEnv}, and \texttt{TradingEnv} tasks. Shared parameters are aligned, while specific tuning differences are highlighted.}
    \label{tab:hyperparameters_comparison}
    \vspace{0.2cm}
    \setlength{\tabcolsep}{3.5pt}
    \begin{tabular}{lccccccccc}
        \toprule
         & \multicolumn{2}{c}{\textbf{ Taxi}} & \multicolumn{2}{c}{\textbf{ Elevator}} & \multicolumn{2}{c}{\textbf{Trading}} \\
        \cmidrule(lr){2-3} \cmidrule(lr){4-5} \cmidrule(lr){6-7} 
        \textbf{Parameter} & \textsc{QL} & \textsc{ExAQ} & \textsc{QL} & \textsc{ExAQ} & \textsc{QL} & \textsc{ExAQ} \\
        \midrule
        \multicolumn{7}{l}{\textit{General Settings}} \\
        \midrule
        Discount Factor ($\gamma$) & $1.0$ & $1.0$ & $1.0$ & $1.0$ & $1.0$ & $1.0$ \\
        Training Seeds & \multicolumn{2}{c}{$[1 \dots 10]$} & \multicolumn{2}{c}{$[1 \dots 10]$} & \multicolumn{2}{c}{$[1 \dots 10]$} \\
        \midrule
        \multicolumn{6}{l}{\textit{Execution}} \\
        \midrule
        Learning Rate ($\alpha$) & $0.05$ & $0.01$ & $0.01$ & $0.001$ & $1.0$ & $0.9$ \\
        Total Episodes & \multicolumn{2}{c}{$15,000$} & \multicolumn{2}{c}{$7,000$} & \multicolumn{2}{c}{$20,000$} \\
        Eval Length & \multicolumn{2}{c}{50 eps} & \multicolumn{2}{c}{50 eps} & \multicolumn{2}{c}{50 eps} \\
        \midrule
        \multicolumn{6}{l}{\textit{Exploration ($\epsilon$-Greedy)}} \\
        \midrule
        Initial Epsilon ($\epsilon_{\text{start}}$) & $1.0$ & $-$ & $1.0$ & $-$ & $1.0$ & $-$  \\
        Minimum Epsilon ($\epsilon_{\text{min}}$) & $0$ & $-$ & $0.05$ & $-$ & $0.05$ & $-$ \\
        Decay Rate & $0.99985$ & $-$ & $0.9995$ & $-$ & $0.9998$ & $-$  \\
        Decay Type & Exp & $-$ & Exp & $-$ & Mixed & $-$  \\
        \bottomrule
    \end{tabular}
\end{table}

\begin{table}[h]
    \centering
    \caption{Hyperparameter settings for \textsc{PPO} in the \texttt{TradingEnv} environment.}
    \label{tab:ppo_hyperparameters}
    \vspace{0.2cm}
    \setlength{\tabcolsep}{8pt} 
    \begin{tabular}{lc}
        \toprule
        \textbf{Parameter} & \textbf{Value} \\
        \midrule
        \multicolumn{2}{l}{\textit{General Settings}} \\
        \midrule
        Discount Factor ($\gamma$) & $0.999$ \\
        Number of Parallel Envs & $16$ \\
        Training Seeds & $[1 \dots 10]$ \\
        \midrule
        \multicolumn{2}{l}{\textit{Optimization}} \\
        \midrule
        Learning Rate ($\alpha$) & $10^{-4}$ \\
        LR Schedule & Linear Annealing \\
        Clipping Parameter ($\epsilon$) & $0.2$ \\
        GAE Parameter ($\lambda$) & $0.95$ \\
        Max Gradient Norm & $0.5$ \\
        Minibatch Size & $512$ \\
        Update Epochs & $4$ \\
        Advantage Normalization & True \\
        \midrule
        \multicolumn{2}{l}{\textit{Loss Coefficients}} \\
        \midrule
        Value Function Coef. ($c_{vf}$) & $0.5$ \\
        Entropy Coef. ($c_{ent}$) & $0.0$ \\
        \midrule
        \multicolumn{2}{l}{\textit{Execution}} \\
        \midrule
        Iterations & $5,000$ \\
        Steps per Iteration & $1,000$ \\
        \bottomrule
    \end{tabular}
\end{table}


\end{document}